\documentclass[twoside]{article}
\usepackage{definitions}
\usepackage[draft]{hyperref}

\usepackage{microtype}
\usepackage{paralist}
\usepackage{multirow}

\usepackage{hyperref}

\usepackage[accepted]{aistats2020}
\usepackage[round]{natbib}

\begin{document}

%

%

\twocolumn[


\aistatstitle{Accelerating Gradient Boosting Machines}

\aistatsauthor{ Haihao Lu$^*$ \And Sai Praneeth Karimireddy$^*$ \And  Natalia Ponomareva \And  Vahab Mirrokni}

 \aistatsaddress{ Google,\\ University of Chicago \And  EPFL \And Google \And Google } 
]

\begin{abstract}
  Gradient Boosting Machine (GBM) introduced by~\cite{friedman2001greedy} is a widely popular ensembling technique and is routinely used in competitions such as Kaggle and the KDDCup~\citep{chen2016xgboost}. In this work, we propose an Accelerated Gradient Boosting Machine (AGBM) by incorporating Nesterov's acceleration techniques into the design of GBM. The difficulty in accelerating GBM lies in the fact that weak (inexact) learners are commonly used, and therefore, with naive application, the errors can accumulate in the momentum term. To overcome it, we design a ``corrected pseudo residual'' that serves as a new target for fitting a weak learner, in order to perform the z-update. Thus, we are able to derive novel computational guarantees for AGBM. This is the first GBM type of algorithm with a theoretically-justified accelerated convergence rate. 
\end{abstract}

\section{Introduction}
Gradient Boosting Machine (GBM) \citep{friedman2001greedy} is an iterative ensembling procedure for supervised tasks (classification or regression) which combines multiple weak-learners to create a strong ensemble. GBM has excellent practical performance and is a staple tool used in Kaggle and the KDDCup \citep{chen2016xgboost}. Its popularity can be attributed to its flexibility---it naturally supports heterogeneous data and tasks---and has several open-source implementations: scikit-learn~\citep{pedregosa2011scikit}, R gbm~\citep{ridgeway2013package}, LightGBM~\citep{ke2017lightgbm}, XGBoost~\citep{chen2016xgboost}, TF Boosted Trees~\citep{ponomareva2017tf}, etc. \blfootnote{\hspace*{-0.2cm}$^*$ Equal contribution.}

Despite its popularity, the theoretical analysis of the method is unsatisfactory. GBM is typically interpreted as an iterative \emph{functional gradient descent}~\citep{mason2000boosting,friedman2001greedy}, but lacks rigorous finite-time convergence guarantees. In this work, we use this viewpoint as a starting point and try to apply well-studied  techniques from first-order convex optimization. 


In convex optimization literature, Nesterov's acceleration is a successful technique used to speed up the convergence of first-order methods. In this work, we show how to incorporate Nesterov momentum into the gradient boosting framework in order to obtain an Accelerated Gradient Boosting Machine (AGBM). This paves the way for speeding up some practical applications of GBMs, which currently require a large number of boosting iterations. For example, GBMs with boosted trees for multi-class problems are often implemented as a number of one-vs-rest learners, resulting in more complicated boundaries \citet{Friedman98additivelogistic} and a potentially a larger number of boosting iterations required. Additionally, it is common practice to build many very-weak learners (for example oblivious trees) for problems where it is easy to overfit. Such large ensembles result not only in slow training time, but also slower inference. AGBMs can be potentially beneficial for all such applications. 

Our main contribution is the first accelerated gradient boosting algorithm that comes \textbf{with strong theoretical guarantees} and which can be used with any type of 
weak learner. We introduce our algorithm in Section \ref{sec-algo} and prove (Section \ref{sec-proof}) that it reduces the empirical loss at a rate of $O(1/m^2)$ after $m$ iterations, improving upon the $O(1/m)$ rate obtained by traditional gradient boosting methods.


{\bf Related Literature.}

{\bf GBM Convergence Guarantees}: After being first introduced by \citet{friedman2001greedy}, several works established its guaranteed convergence, without explicitly stating the convergence rate \citep{collins2002logistic,mason2000boosting}. Subsequently, when the loss function is both smooth and strongly convex, \citet{bickel2006some} proved a slow convergence rate---more precisely that $O(\exp(1/\varepsilon^2))$ iterations are sufficient to ensure that the training loss is within $\varepsilon$ of its optimal value. \citet{telgarsky2012primal} then studied the primal-dual structure of GBM and demonstrated that in fact only $O(\log(1/\varepsilon))$ iterations are needed. However the constants in their rate were non-standard and less intuitive. This result was recently improved upon by \citet{freund2017new} and \citet{lu2018randomized}, who showed a similar convergence rate but with more transparent constants such as the smoothness and strong convexity constant of the loss function, as well as the density of weak learners. 
Additionally, if the loss function is smooth and convex (not necessarily strongly convex), \citet{lu2018randomized} showed that $O(1/\varepsilon)$ iterations suffice.
Please refer to \citet{telgarsky2012primal}, \citet{freund2017new}, \citet{lu2018randomized} for a review of theoretical results of GBM convergence.

{\bf Accelerated Gradient Methods}:
For optimizing a smooth convex function, \citet{nesterov1983method} showed that the standard gradient descent (GD) algorithm can be made much faster, resulting in the \emph{accelerated} gradient descent method. While GD requires $O(1/\varepsilon)$ iterations, accelerated gradient methods only require $O(1/\sqrt{\varepsilon})$. This rate of convergence is optimal and cannot be improved upon \citet{nesterov2004introductory}. The mainstream research community's interest in Nesterov's method started only around 15 years ago; yet even today most researchers struggle to find basic intuition as to what is really going on in accelerated methods.  
Such lack of intuition about the estimation sequence proof technique used by \citet{nesterov2004introductory} has motivated many recent works trying to explain this acceleration phenomenon \citep{su2016differential, wilson2016lyapunov, hu2017dissipativity, lin2015universal, frostig2015regularizing, allen2014linear, bubeck2015geometric, chambolle2015convergence}. There are also attempts to give a physical explanation of acceleration by studying the continuous-time interpretation of accelerated GD via dynamical systems \citep{su2016differential, wilson2016lyapunov, hu2017dissipativity}. 


{\bf Accelerated Greedy Coordinate Descent and Matching Pursuit Methods}: 
GBM can be viewed as a greedy coordinate descent algorithm or a matching pursuit algorithm in transformed spaces. Recently,  ~\citet{lu2018accelerating} and \citet{locatello2018matching} discussed how to accelerate greedy coordinate descent and matching pursuit algorithms respectively. Their methods however require a random step and are hence only `semi-greedy', which does not fit in to the boosting framework. 


{\bf Accelerated GBM}: Very recently, \citet{biau2018accelerated} and \citet{fouillen2018accelerated} proposed accelerated versions of GBM by directly incorporating Nesterov's momentum in GBM, but without theoretical justification. Furthermore, as we show in Section \ref{subsec:vagm-diverge}, their proposed algorithm \textbf{may not converge} to the optimum.

\section{Gradient Boosting Machine}\label{sec:gradient-boosting}

We consider a supervised learning problem with $n$ training examples
$(x_{i},y_{i}), i = 1,\ldots,n $
such that $x_{i} \in \RR^p$ is the feature vector of the $i$-th example and $y_{i}$ is a label (in a classification problem) or a continuous response (in a regression problem). In the classical version of GBM~\citep{friedman2001greedy}, we assume we are given a base class of \emph{learners} $\BB$ and that our target function class is the linear combination of such base learners (denoted by $\lin(\BB)$). Let $\BB = \{b_{\tau}(x) \in \RR\}$ be a family of learners parameterized by $\tau \in \wl$. The prediction corresponding to a feature
vector $x$ is given by an additive model of the form:
\begin{equation}\label{eq:add-model}
f(x):=\left(\sum_{m=1}^{M}\beta_{m}b_{\tau_m}(x)\right) \in \lin(\BB)\ ,
\end{equation}
where $b_{\tau_m}(x) \in \BB$ is a weak-learner and $\beta_{m}$ is its corresponding additive coefficient. Here, $\beta_{m}$ and $\tau_{m}$ are chosen in an adaptive fashion in order to improve the data-fidelity as discussed below. Examples of learners commonly used in practice include wavelet functions, support vector machines, and classification and regression trees \citep{friedman2001elements}. 
We assume the set of weak learners $\BB$ is \emph{scalable}, namely that the following assumption holds.
\begin{ass}\label{ass:scalable}
If $b(\cdot)\in\BB$, then $\lambda b(\cdot)\in \BB$ for any $\lambda>0$.
\end{ass}
Assumption \ref{ass:scalable} holds for most of the set of weak learners we are interested in. Indeed scaling a weak learner is equivalent to modifying the coefficient of the weak learner, so it does not change the structure of $\BB$.

The goal of GBM is to obtain a good estimate of the function $f$ that approximately minimizes the empirical loss:
\begin{equation}
L^\star = \min_{f \in \lin(\BB)}\Big\{L(f) := \sum_{i=1}^{n}\ell(y_{i},f(x_{i})\Big\}\,\label{eq:loss}
\end{equation}
where $\ell(y_{i}, f(x_{i}))$ is a measure of the data-fidelity for the $i$-th sample
for the loss function $\ell$.

\subsection{Best Fit Weak Learners}
The original version of GBM by ~\citep{friedman2001greedy}, presented in Algorithm \ref{al:gbm}, can be viewed as minimizing the loss function by applying an approximated steepest descent algorithm to the loss in~\eqref{eq:loss}.
GBM starts from a null function $f^0\equiv0$ and at each iteration $m$ computes the pseudo-residual $r^m$ (namely, the negative gradient of the loss function with respect to the predictions so far $f^m$), then a weak-learner that best fits the current pseudo-residual in terms of the least squares loss is computed.

This weak-learner is added to the model with a coefficient found via a line search. As the iterations progress, GBM leads to a sequence of functions
$\{f^m\}_{m \in [M]}$ (where $[M]$ is a shorthand for the set $\{1,\ldots,M\}$). The usual intention of GBM is to stop 
early---before one is close to a minimum of Problem~\eqref{eq:loss}---with the hope that such a model will lead to good predictive performance~\citep{friedman2001greedy,freund2017new, zhang2005boosting, buhlmann2007boosting}.
\begin{algorithm}[h]
\caption{Gradient Boosting Machine (GBM)~\citep{friedman2001greedy}}\label{al:gbm}

\begin{algorithmic}
\STATE {\bf Initialization.}  Initialize with $f^{0}(x)=0$.\\
For $m=0,\ldots,M-1$ do:\\

\STATE  {\bf Perform Updates:}  

(1) Compute pseudo residual: $r^m=-\left[\frac{\partial \ell(y_{i},f^{m}(x_{i}))}{\partial f^m(x_{i})}\right]_{i=1,\ldots,n}\,.$

(2) Find the parameters of the best weak-learner: $\tau_m =\argmin_{\tau \in \wl} \sum_{i=1}^n (r_i^m- b_{\tau}(x_i))^{2}\,.$

(3) Choose the step-size  $\eta_m$ by line-search: $\eta_{m}=\argmin_{\eta}\sum_{i=1}^{n}\ell(y_{i},f^{m}(x_{i})+\eta b_{\tau_{m}}(x_i))\,.$

(4) Update the model $f^{m+1}(x)=f^{m}(x)+\eta_{m}b_{\tau_{m}}(x)$.\\
\medskip

\STATE  {\bf Output.}  $f^{M}(x)$.
\end{algorithmic}
\end{algorithm}\medskip

Perhaps the most popular set of learners are classification and regression trees (CART) \citep{breiman2017classification}, resulting in Gradient Boosted Decision Tree models (GBDTs). We will use GBDTs for our numerical experiments. At the same time, we would like to highlight that our algorithm is not tied to a particular type of a weak learner and is a general algorithm.

\section{Accelerated Gradient Boosting Machine (AGBM)}\label{sec-algo}
Given the success of accelerated GD as a first order optimization method, it seems natural to attempt to accelerate the GBMs. To start, we first look at how to obtain an accelerated boosting algorithm when our class of learners $\BB$ is \textbf{strong} (i.e. complete) and can exactly fit any pseudo-residuals. This assumption is quite \emph{unreasonable} but will serve to understand the connection between boosting and first order optimization. We then proceed with an algorithm that works for any class of \textbf{weak} learners.

\subsection{First Attempt: Boosting with strong learners}\label{sec:strong}
\begin{table*}
    \centering
    \begin{tabular}{|c|c|c|}
    \hline
    Parameter     & Dimension & Explanation  \\
    \hline
    $(x_i,y_i)$ &$\RR^p\times\RR$ & The features and the label of the $i$-th sample.\\
    \hline
    $X$ & $\RR^{p\times n}$ & $X=[x_1,x_2,\cdots,x_n]$ is the feature matrix for all training data. \\
    \hline
    $b_{\tau}(x)$     & function & Weak learner parameterized by $\tau$. \\
    \hline
    $b_{\tau}(X)$ & $\RR^n$ & A vector of predictions $[b_{\tau}(x_i)]_i$. \\
    \hline
    $f^m(x)$     & function & Ensemble of weak learners at the $m$-th iteration. \\
    \hline
    $f(X)$ & $\RR^n$ & A vector of $[f(x_i)]_i$ for any function $f(x)$. \\
    \hline
    $g^m(x), h^m(x)$     & functions & Auxiliary ensembles of weak learners at the $m$-th iteration. \\
    \hline
    $r^m$     & $\RR^n$ & Pseudo residual at the $m$-th iteration. \\
    \hline
    $c^m$     & $\RR^n$ & Corrected pseudo-residual at the $m$-th iteration. \\
    \hline

    \end{tabular}
    \caption{List of notations used.}
    \label{tab:notations} \vspace*{-0.3cm}
\end{table*}
In this subsection, we assume the class of learners $\BB$ is \emph{strong}, i.e. for any pseudo-residual $r \in \RR^n$, there exists a learner $b(x) \in \BB$ such that 
$
b(x_i) = r_i, \forall \ i \in [n]\,.
$
Of course the entire point of boosting is that the learners are \emph{weak} and thus the class is not strong, therefore this is not a realistic assumption. Nevertheless this section will provide the intuitions on how to develop AGBM.

In the GBM we compute the psuedo-residual $r^m$ to be the negative gradient of the loss function over the predictions so far. A gradient descent step in a functional space would define $f^{m+1}$ as, for $i\in\{1,\ldots,n\}$,
$
    f^{m+1}(x_i) = f^m(x_i) + \eta r^m_i\,.
$
Here $\eta$ is the step-size of our algorithm. Since our class of learners is rich, we can choose $b^m(x) \in \BB$ to exactly satisfy the above equation.

Thus GBM (Algorithm \ref{al:gbm}) then has the following update:
\[
    f^{m+1} = f^{m} + \eta b^m\,,
\]
where $b^m(x_i) = r^m_i$.
In other words, GBM performs exactly functional gradient descent when the class of learners is strong, and so it converges at a rate of $O(1/m)$. Akin to the above argument, we can perform functional \emph{accelerated} gradient descent, which has the accelerated rate of $O(1/m^2)$. In the accelerated method, we maintain three model ensembles: $f$, $g$, and $h$ of which $f(x)$ is the only model which is finally used to make predictions during the inference time. Ensemble $h(x)$ is the \emph{momentum} sequence and $g(x)$ is a weighted average of $f$ and $h$ (refer to Table \ref{tab:notations} for list of all notations used). These sequences are updated as follows for a step-size $\eta$ and $\{\theta_m = 2/(m+2)\}$:
\begin{equation}
    \begin{split}
        g^{m} &= (1- \theta_m)f^m + \theta_m h^m\\
        f^{m+1} &= g^m + \eta b^m \hspace*{1cm} \text{: primary model}\\
        h^{m+1} &= h^m + {\eta}/{\theta_m}b^m \quad \text{: momentum model}\,
    \end{split}\label{eq:(1)accel-GD-updates}
\end{equation}
where $b^m(x)$ satisfies for $i \in 1,\dots,n$
\begin{equation}
    b^m(x_i) = - \frac{\partial \ \ell(y_i, g^m(x_i))}{\partial g^m(x_i)}\,. \label{eq:residual-g}
\end{equation}
Note that the psuedo-residual is computed w.r.t. $g$ instead of $f$. The update above can be rewritten as 
\[
    f^{m+1} = f^m + \eta b^m + \theta_m (h^m - f^m)\,.
\]
If $\theta_m = 0$, we see that we recover the standard functional gradient descent with step-sze $\eta$. For $\theta_m \in (0,1]$, there is an additional \emph{momentum} in the direction of $(h^m - f^m)$.

The three sequences $f$, $g$, and $h$ match exactly those used in typical accelerated gradient descent methods (see \citet{nesterov2004introductory,tseng2008accelerated} for details).

\subsection{Main Setting: Boosting with weak learners}\label{sec-weak-learners}
In this subsection, we consider the general case without assuming that the class of learners is strong. Indeed, the class of learners $\BB$ is usually quite simple and it is very likely that for any $\tau \in \wl$, it is impossible to exactly fit the residual $r^m$.
We call this case boosting with \textbf{weak} learners. Our task then is to modify \eqref{eq:(1)accel-GD-updates} to obtain a truly accelerated gradient boosting machine.

The full details are summarized in Algorithm \ref{al:agbm} but we will highlight two key differences from \eqref{eq:(1)accel-GD-updates}. 
First, the update to the $f$ sequence is replaced with a weak-learner which best approximates $r^m$ similar to step 2 of Algorithm \ref{al:gbm}. In particular, we compute pseudo-residual $r^m$ computed w.r.t. $g$ as in \eqref{eq:residual-g} and find a parameter $\tau_{m,1}$ such that
$
    \tau_{m,1} = \argmin_{\tau \in \wl}\sum_{i=1}^n(r_i^m - b_\tau(x_i))^2\,.
$

Secondly, and more crucially, the update to the momentum model $h$ is decoupled from the update to the $f$ sequence. We use an \emph{error-corrected} pseudo-residual $c^m$ instead of directly using $r^m$. Suppose that at iteration $m-1$, a weak-learner $b_{\tau_{m-1,2}}$ was added to $h^{m-1}$. Then error corrected residual is defined inductively as follows: for $i \in \{1,\dots,n\}$
\[
    c^m_i = r^m_i + \frac{m+1}{m+2}\left(c^{m-1}_i - b_{\tau_{m-1,2}}(x_i)\right)\,,
\]
and then we compute
\[
    \tau_{m,2} = \argmin_{\tau \in \wl}\sum_{i=1}^n(c_i^m - b_\tau(x_i))^2\,.
\]
Thus at each iteration \emph{two} weak learners are computed---$b_{\tau_{m,1}}(x)$ approximates the residual $r^m$ and the $b_{\tau_{m,2}}(x)$, which approximates the error-corrected residual $c^m$.
Note that if our class of learners is complete (i.e. strong), then $c_i^{m-1} = b_{\tau_{m-1,2}}(x_i)$, $c^m = r^m$ and $\tau_{m,1} = \tau_{m,2}$. This would revert back to our accelerated gradient boosting algorithm for strong-learners described in \eqref{eq:(1)accel-GD-updates}.

The difficulty of accelerating boosting with weak learners is the error made during the fitting of the pseudo-residual. The momentum term ends up taking large steps, thereby `amplifying’ and accumulating such error. Using error-correction helps to correct the error from the past steps, ensuring that the error remains controlled. 

In Algorithm \ref{al:agbm}, we utilize a new hyper-parameter $\gamma$, which is called momentum-parameter throughout the paper.
In practice, the performance of AGBM is not too sensitive to the momentum-parameter $\gamma$ and can be manually picked to lie between $(0,1)$.

\begin{algorithm*}
\caption{Accelerated Gradient Boosting Machine (AGBM)}\label{al:agbm}
\begin{algorithmic}
\STATE {\bf Input.} Starting function $f^0(x)=0$, step-size $\eta$, momentum-parameter $\gamma \in (0,1]$, and data $X,y = (x_i,y_i)_{i\in [n]}$.
\STATE {\bf Initialization.}  $h^{0}(x)= f^{0}(x)$ and sequence $\theta_m=\frac{2}{m+2}$.\\
For $m=0,\ldots,M-1$ do:\\

\STATE  {\bf Perform Updates:}  

(1) Compute a linear combination of $f$ and $h$: $g^{m}(x) = (1 - \theta_m)f^{m}(x) + \theta_m h^{m}(x)$.

(2) Compute pseudo residual: $r^m=-\left[\frac{\partial \ell(y_{i},g^{m}(x_{i}))}{\partial g^{m}(x_{i})}\right]_{i=1,\ldots,n}.$

(3) Find the best weak-learner for pseudo residual: $\tau_{m,1} =\argmin_{\tau \in \wl}\sum_{i=1}^n ( r_i^m- b_\tau(x_i))^{2}$.

(4) Update the model: $f^{m+1}(x)=g^{m}(x)+\eta b_{\tau_{m,1}}(x)$.

(5) Update the corrected residual: $c^m_i =     \begin{cases}
                                                    r^m_i &\text{ if } m =0 \\
                                                    r^m_i + \frac{m+1}{m+2}(c^{m-1}_i - b_{\tau_{m-1,2}}(x_i)) &\text{ o.w.}
                                                \end{cases}$.

(6) Find the best weak-learner for the corrected residual: $\tau_{m,2}  =\argmin_{\tau \in \wl} \sum_{i=1}^n (c^m_i - b_\tau(x_i))^{2}$.

(7) Update the momentum model: $h^{m+1}(x) = h^m(x) + \gamma \eta/\theta_m b_{\tau_{m,2}}(x)$.
\medskip

\STATE  {\bf Output.}  $f^{M}(x)$.
\end{algorithmic}
\end{algorithm*}\medskip

\section{Convergence Analysis of AGBM}\label{sec-proof}
We first list the assumptions required and then outline the computational guarantees for AGBM.
\subsection{Assumptions}
Let's introduce some standard regularity/continuity constraints on the loss that we use in our analysis.
\begin{mydef}
We denote $\frac{\partial\ell(y,f)}{\partial f}$ as the derivative of the bivariant loss function $\ell$ w.r.t. the prediction $f$. We say that $\ell$ is $\sigma$-smooth if for any $y$ and scalar predictions $f_1$ and $f_2$, it holds that
$$
\ell(y, f_1) \le \ell(y, f_2) + \frac{\partial \ell (y, f_2)}{\partial f}(f_1-f_2) + \frac{\sigma}{2} (f_1-f_2)^2 .
$$
We say $\ell$ is $\mu$-strongly convex (with $\mu>0$) if for any $y$ and predictions $f_1$ and $f_2$, it holds that
$$
\ell(y, f_1) \ge \ell(y, f_2) +  \frac{\partial \ell (y, f_2)}{\partial f}(f_1-f_2) + \frac{\mu}{2} (f_1-f_2)^2 .
$$
\end{mydef}
Note that $\mu \leq \sigma$ always. Smoothness and strong-convexity mean that the function $l(x)$ is upper and lower bounded by quadratic functions. Intuitively, smoothness implies that gradient does not change abruptly and $l(x)$ is never `sharp'. Strong-convexity implies that $l(x)$ always has some `curvature' and is never `flat'.

The notion of Minimal Cosine Angle (MCA) introduced in \citet{lu2018randomized} plays a central role in our convergence rate analysis of GBM. 
MCA measures how well the weak-learner $b_{\tau}(X)$ approximates the desired residual $r$:
\begin{mydef}\label{def:MCA}
Let $r \in \RR^n$ be a vector. The Minimal Cosine Angle (MCA) is defined as the similarity between $r$ and the output of the best-fit learner $b_{\tau}(X)$:
\begin{equation}\label{eq:Theta}
    \Theta:=\min_{r \in \RR^n} \max_{\tau \in \wl} \cos(r, b_{\tau}(X)) \,,
\end{equation}
where $b_{\tau}(X) \in \RR^n$ is a vector of predictions $[b_\tau(x_i)]_i$.
\end{mydef}


The quantity $\Theta \in (0,1]$ measures how ``dense'' the learners are in the prediction space. For \emph{strong} learners (in Section \ref{sec:strong}), the prediction space is complete, and $\Theta=1$. For a complex space of learners $\wl$ such as deep trees, we expect the prediction space to be dense and that $\Theta \approx 1$. For a simpler class such as tree-stumps $\Theta$ would be much smaller.

It is also straightforward to extend the definition of $\Theta$ (and hence all our convergence results) to approximate fitting of weak learners. Such a relaxation is necessary since computing the exact best-fit weak learner is often computationally prohibitive.
 Refer to \citet{lu2018randomized} for a more thorough discussion of $\Theta$.
\subsection{Computational Guarantees}
We are now ready to state the main theoretical result of our paper.
\begin{thm}\label{thm:agbm-rate-main}
Consider the Accelerated Gradient Boosting Machine (Algorithm \ref{al:agbm}). Suppose $\ell$ is $\sigma$-smooth, the step-size $\eta\le \frac{1}{\sigma}$ and the momentum parameter $\gamma\le \Theta^4/(4 + \Theta^2)$, where $\Theta$ is the MCA introduced in Definition \ref{def:MCA}. Then for all $M\ge 0$, we have:
$$
L(f^M) - L(f^*) \le \frac{1}{2\eta \gamma (M + 1)^2} \|f^*(X)\|_2^2 \ .
$$

\end{thm}
\begin{proof}[Proof Sketch]
Here we only give an outline---the full proof can be found in the Appendix (Section \ref{sec:agbm-rate}). We use the potential-function based analysis of accelerated method (cf. \citet{tseng2008accelerated, wilson2016lyapunov}). Recall that $\theta_m = \frac{2}{m+2}$. 
For the proof, we introduce the following vector sequence of auxiliary ensembles $\hat h$ as follows: 
\[  \hat h^0(X) = 0,~~ 
	\hat h^{m+1}(X) =  \hat h^{m}(X) +\frac{\eta \gamma}{\theta_m}r^m\,.
\]
The sequence $\hat h^{m}(X)$ is in fact closely tied to the sequence $h^{m}(X)$ as we demonstrate in the Appendix (Lemma \ref{lem:L-decrease}). Let $f^\star$ be any function which obtains the optimal loss \eqref{eq:loss}
\[
f^\star \in \argmin_{f \in \lin(\BB)}\Big\{L(f) := \sum_{i=1}^{n}\ell(y_{i},f(x_{i}))\Big\}\,.
\]
Let us define the following sequence of potentials:
\[
	 V^{m}(f^\star) = \begin{cases}
	    \frac{1}{2}\norm*{f^\star(X) - \hat h^0(X)}^2 \quad\quad \text{ if } m=0\,,\\
	    \frac{\eta\gamma}{\theta_{m-1}^2} \encaser{L(f^{m}) - L^\star} + \frac{1}{2}\norm*{f^\star(X) - \hat h^m(X)}^2  \text{ o.w }
	 \end{cases}
\]
Typical proofs of accelerated algorithms show that the potential $V^m(f^\star)$ is a \emph{decreasing} sequence. In boosting, we use the weak-learner that fits the pseudo-residual of the loss. This can guarantee sufficient decay to the first term of $V^m(f^\star)$ related to the loss $L(f)$. However, there is no such guarantee that the same weak-learner can also provide sufficient decay to the second term as we do not apriori know the optimal ensemble $f^\star$. That is the major challenge in the development of AGBM. 

We instead show that the potential decreases at least by $\delta_m$:
\[
	V^{m+1}(f^\star)\leq V^{m}(f^\star) + \delta_m\,,
\]
where $\delta_m$ is an error term depending on $\Theta$ that can be negative (see Lemma \ref{lem:potential-decrease} for the exact definition of $\delta_m$ and proof of the claim). By telescope, it holds that
\begin{align*}
	\frac{\eta\gamma}{\theta_{m}^2} \encaser{L(f^{m+1}) - L (f^\star)} &\leq V^{m+1}(f^\star) \\
	&\leq \sum_{j=0}^m\delta_j + \frac{1}{2}\norm*{f^\star(X) - \hat h^0(X)}^2\,.
\end{align*}
Finally a careful analysis of the error term (Lemma \ref{lem:computations}) shows that $\sum_{j=0}^m\delta_j \leq 0$ for any $m \geq 0$. Therefore,
\[
	L(f^{m+1}) - L (f^\star) \leq \frac{\theta_m^2}{2\eta\gamma}\norm*{f^\star(X)}^2\,,
\]
which finishes the proof by letting $m=M-1$ and substituting the value of $\theta_m$.
\end{proof}
\begin{rem}
Theorem \ref{thm:agbm-rate-main} implies that to get a function $f^M$ such that the error $L(f^M) - L(f^\star) \leq \varepsilon$, we need number of iterations
$
	M = O\left(\frac{1}{\Theta^2 \sqrt{\varepsilon}}\right)\
$.
In contrast, standard gradient boosting machines, as proved in \citet{lu2018randomized}, require
$
	M = O\left(\frac{1}{\Theta^2 \varepsilon}\right)\,
$
This means that for small  values of $\varepsilon$, AGBM (Algorithm \ref{al:agbm}) can require far fewer weak learners than GBM (Algorithm \ref{al:gbm}).
\end{rem}

The next Theorem presents an \emph{accelerated} linear rate of convergence by restarting Algorithm \ref{al:agbm} for minimizing strongly convex loss function $l(x)$.
\begin{algorithm}[!h]
\caption{Accelerated Gradient Boosting Machine with Restart (AGBMR)}\label{al:agbmr}
\begin{algorithmic}
\STATE {\bf Input:} Starting function $\tilde f^0(x)$, step-size $\eta$, momentum-parameter $\gamma \in (0,1]$, strong-convexity parameter $\mu$.\\
For $p=0,\ldots,P-1$ do:\\ 
\STATE \hspace*{0.2cm} (1) Run AGBM (Algorithm \ref{al:agbm}) initialized with $f^0(x) = \tilde f^p(x)$:
\STATE \hspace*{0.5cm} \textbf{Option 1:}  for $M = \sqrt{\frac{2}{\eta\gamma \mu}}$ iterations.\\
\STATE \hspace*{0.5cm} \textbf{Option 2:}  until $L(f^{m}) > L(f^{m-1})$.\\
\STATE \hspace*{0.2cm} (2) Set $\tilde f^{p+1}(x) = f^{M}(x)$.
\STATE {\bf Output:} $\tilde f^{P}(x)$.
\end{algorithmic}
\end{algorithm}\medskip

\begin{thm}\label{thm:agbmr-lin-rate}
Consider Accelerated Gradient Boosting with Restarts with Option 1 (Algorithm \ref{al:agbmr}) . Suppose that  $l(x)$ is $\sigma$-smooth and $\mu$-strongly convex. If the step-size $\eta\le \frac{1}{\sigma}$ and the momentum parameter $\gamma\le \Theta^4/(4 + \Theta^2)$, then for any $p$ and optimal loss $L(f^\star)$,
\[
	L(\tilde f^{p+1}) - L^\star \leq \frac{1}{2}(L(\tilde f^p) - L(f^\star))\,.
\]
\end{thm}
\begin{proof}
	The loss function $l(x)$ is $\mu$-strongly convex, which implies that
	\[
		\frac{\mu}{2}\|f(X) -f^*(X)\|_2^2 \leq L(f) - L(f^\star)\,.
	\]
	Substituting this in Theorem \ref{thm:agbm-rate-main} gives us that
	\[
		L(f^{M}) - L(f^\star) \leq \frac{1}{\mu \eta \gamma (M+1)^2}(L(f^0) - L(f^\star))\,.
	\]
	Recalling that $f^0(x) = \tilde f^p(x)$, $f^{M}(x)= \tilde f^{p+1}(x)$, and $M^2 = {2}/{\eta\mu\gamma}$ gives us the required statement.
\end{proof}
The restart strategy in \emph{Option 1} requires knowledge of the strong-convexity constant $\mu$. Alternatively, one can also use adaptive restart strategy (\emph{Option 2}) which is known to have good empirical performance \citep{o2015adaptive}.
\begin{rem}
Theorem \ref{thm:agbmr-lin-rate} shows that 
$
M = O\encaser{\frac{1}{\Theta^2} \sqrt{\frac{\sigma}{\mu}}\log(1/\varepsilon)}
$ 
weak learners are sufficient to obtain an error of $\varepsilon$ using ABGMR (Algorithm \ref{al:agbmr}). In contrast, standard GBM (Algorithm \ref{al:gbm}) requires
$
M = O\encaser{\frac{1}{\Theta^2}{\frac{\sigma}{\mu}}\log(1/\varepsilon)}
$
weak learners. Thus AGBMR is significantly better than GBM only if the condition number is large i.e. $(\sigma/\mu \ge 1)$. When $l(y,f)$ is the least-squares loss, $(\mu = \sigma = 1)$ we would see no advantage of acceleration. However for more complicated functions with $(\sigma \gg \mu)$ (e.g. logistic loss or exp loss), AGBMR might result in an ensemble that is significantly better (e.g. obtaining lower training loss) than that of GBM for the same number of weak learners.
\end{rem}

\section{Numerical Experiments}
\begin{figure*}
\centering
{\begin{tabular}{l c c c}
&{   {$\eta=1$}}&{   {$\eta=0.1$}}&{   {$\eta=0.01$}}\\

\rotatebox{90}{ {~~~train loss}}&
\includegraphics[width=0.2\textwidth,   trim =1cm 0.6cm 1.0cm 1.0cm, clip = true]{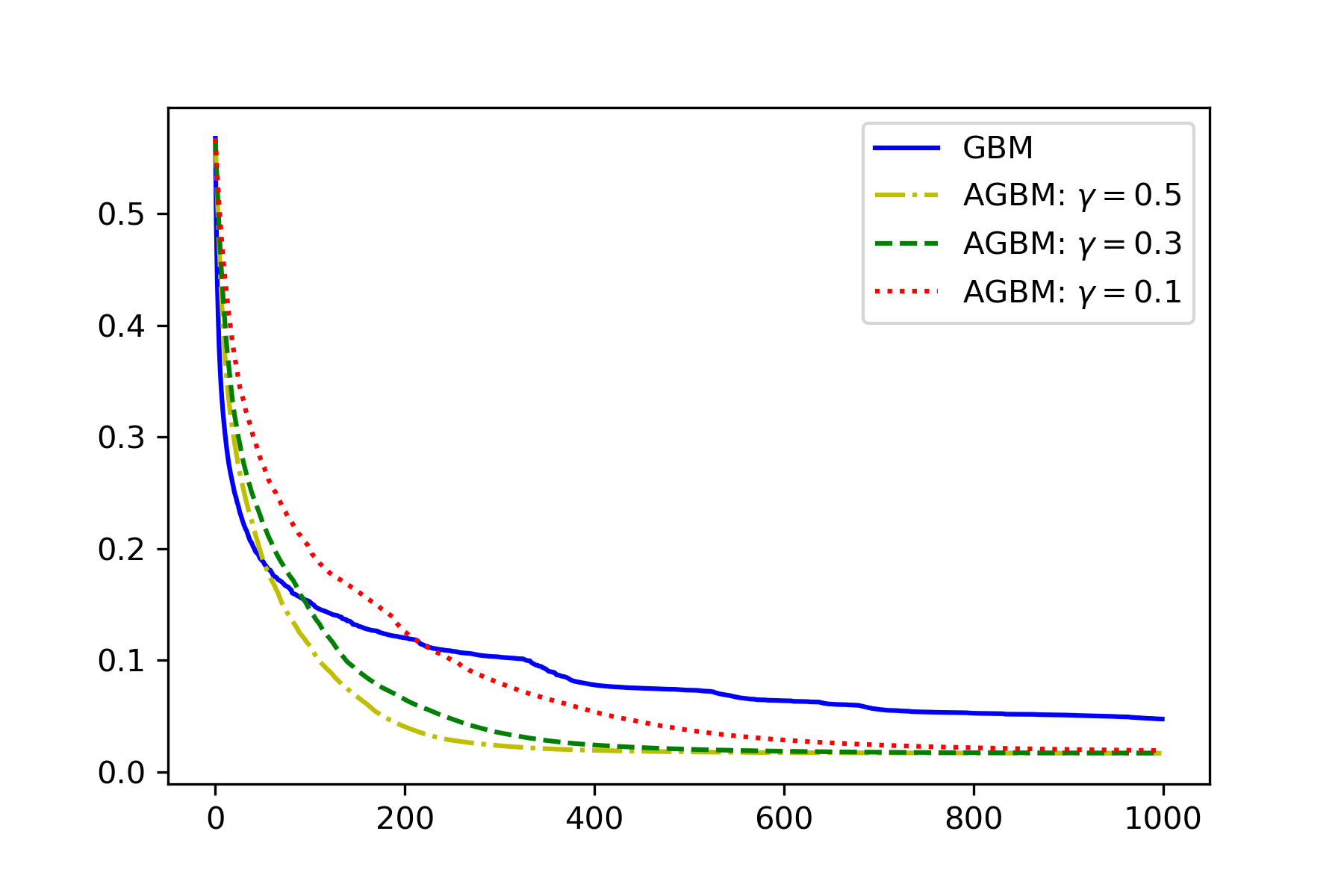}&
\includegraphics[width=0.2\textwidth,  trim =1cm 0.6cm 1.0cm 1.0cm, clip = true]{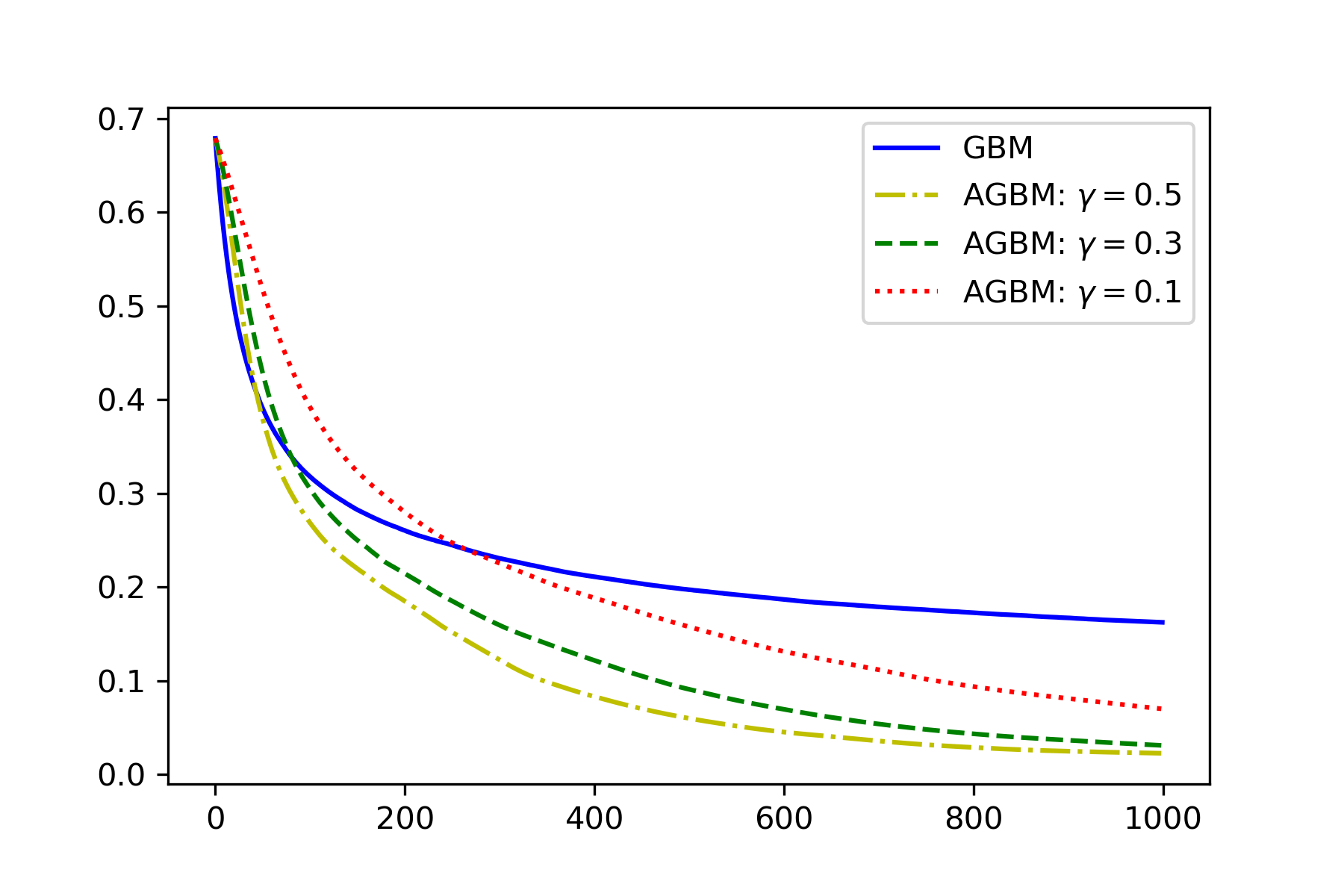} &
\includegraphics[width=0.2\textwidth,  trim =1cm 0.6cm 1.0cm 1.0cm, clip = true]{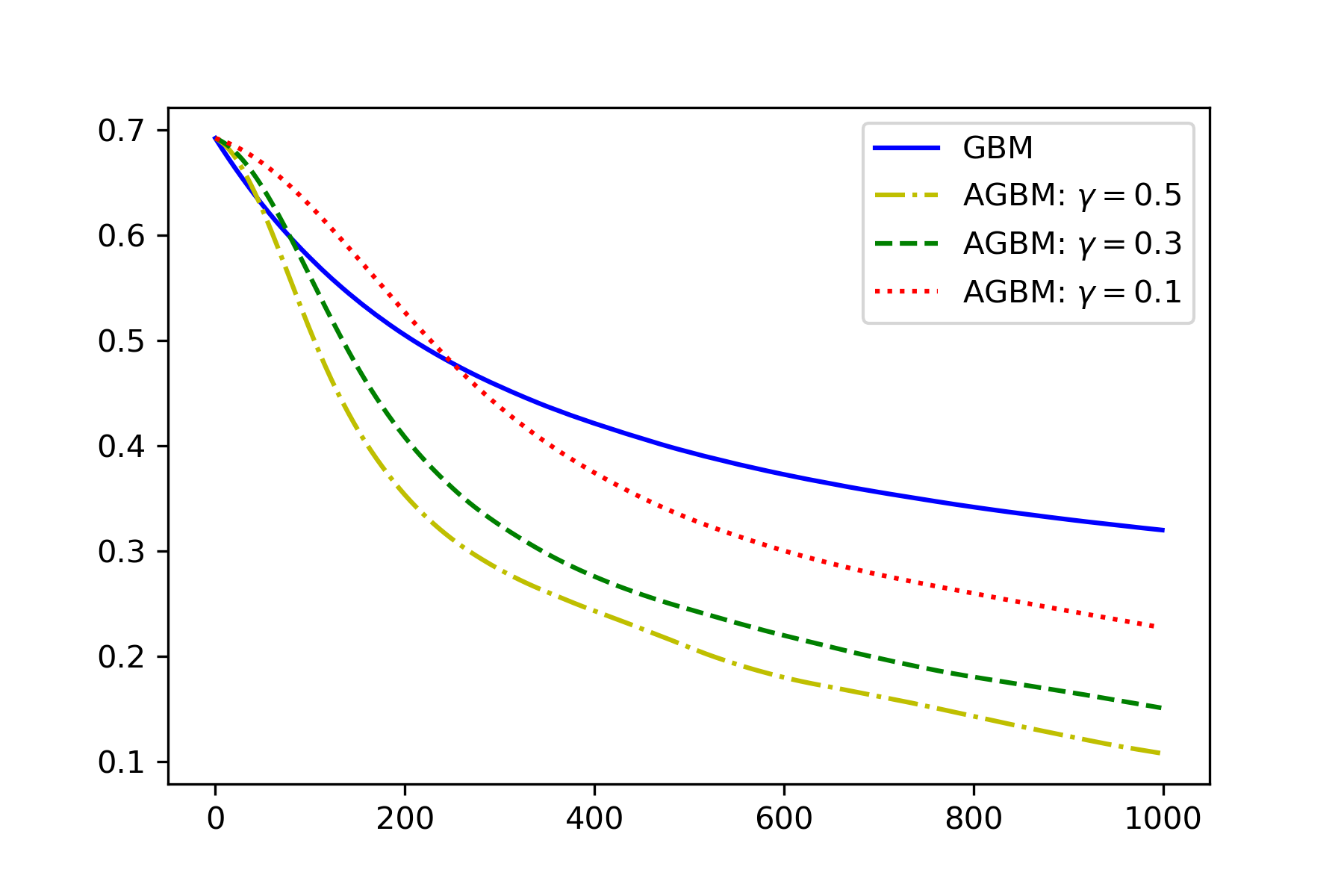} \\

\rotatebox{90}{ {~~~~test loss}}&
\includegraphics[width=0.2\textwidth, trim =1cm 0.6cm 1.0cm 1.0cm, clip = true]{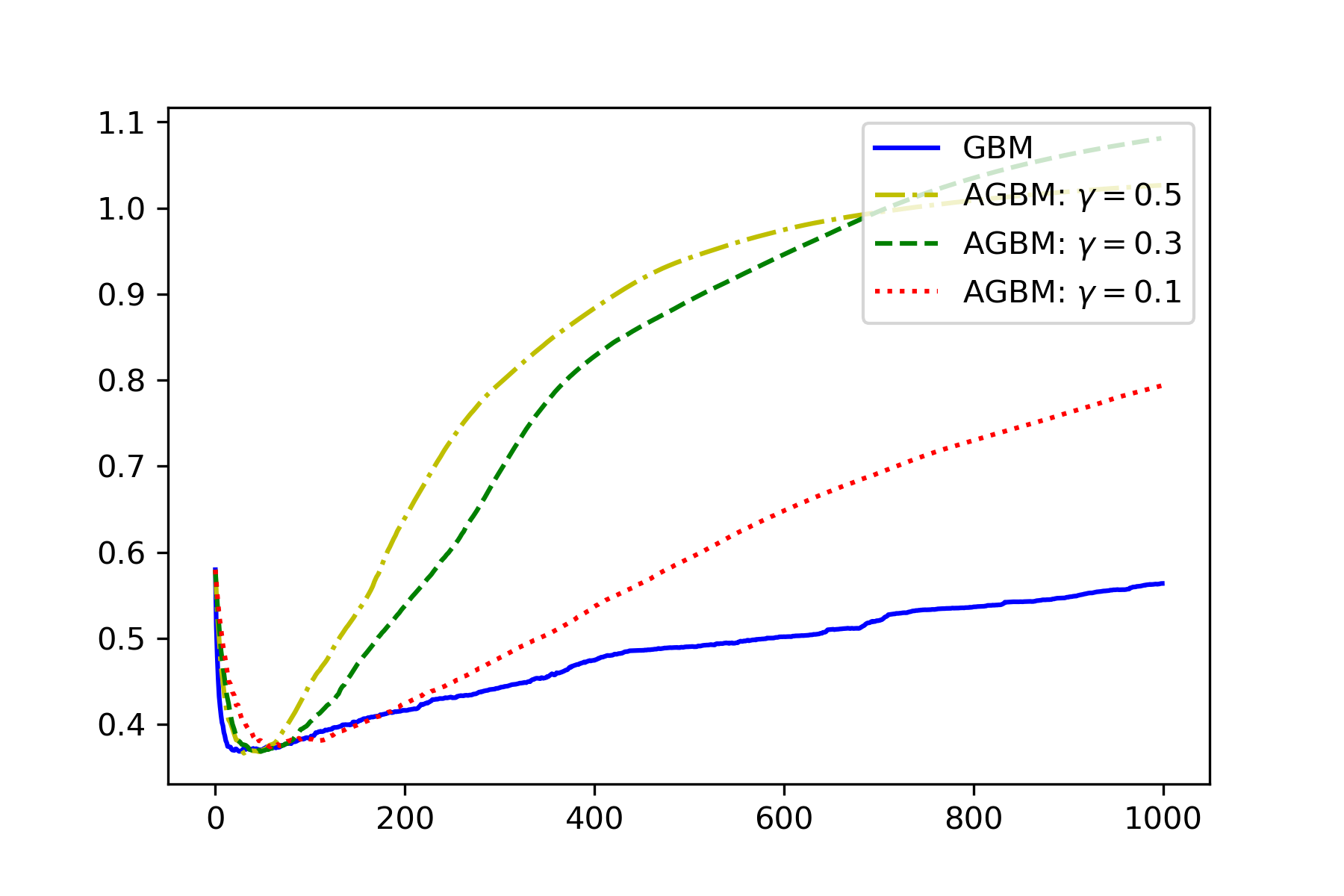}&

\includegraphics[width=0.2\textwidth, trim =1cm 0.6cm 1.0cm 1.0cm, clip = true]{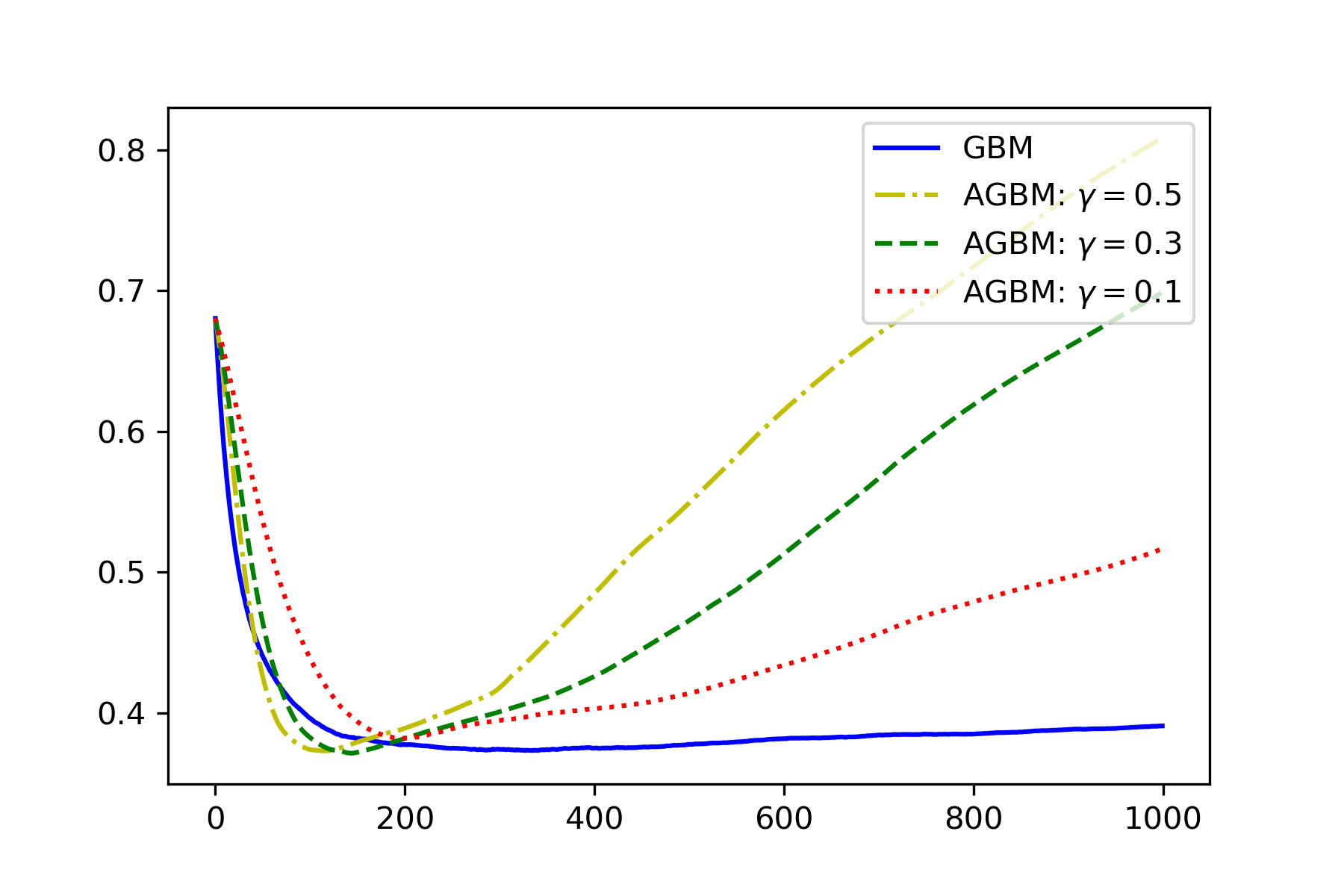}&

\includegraphics[width=0.2\textwidth, trim =1cm 0.6cm 1.0cm 1.0cm, clip = true]{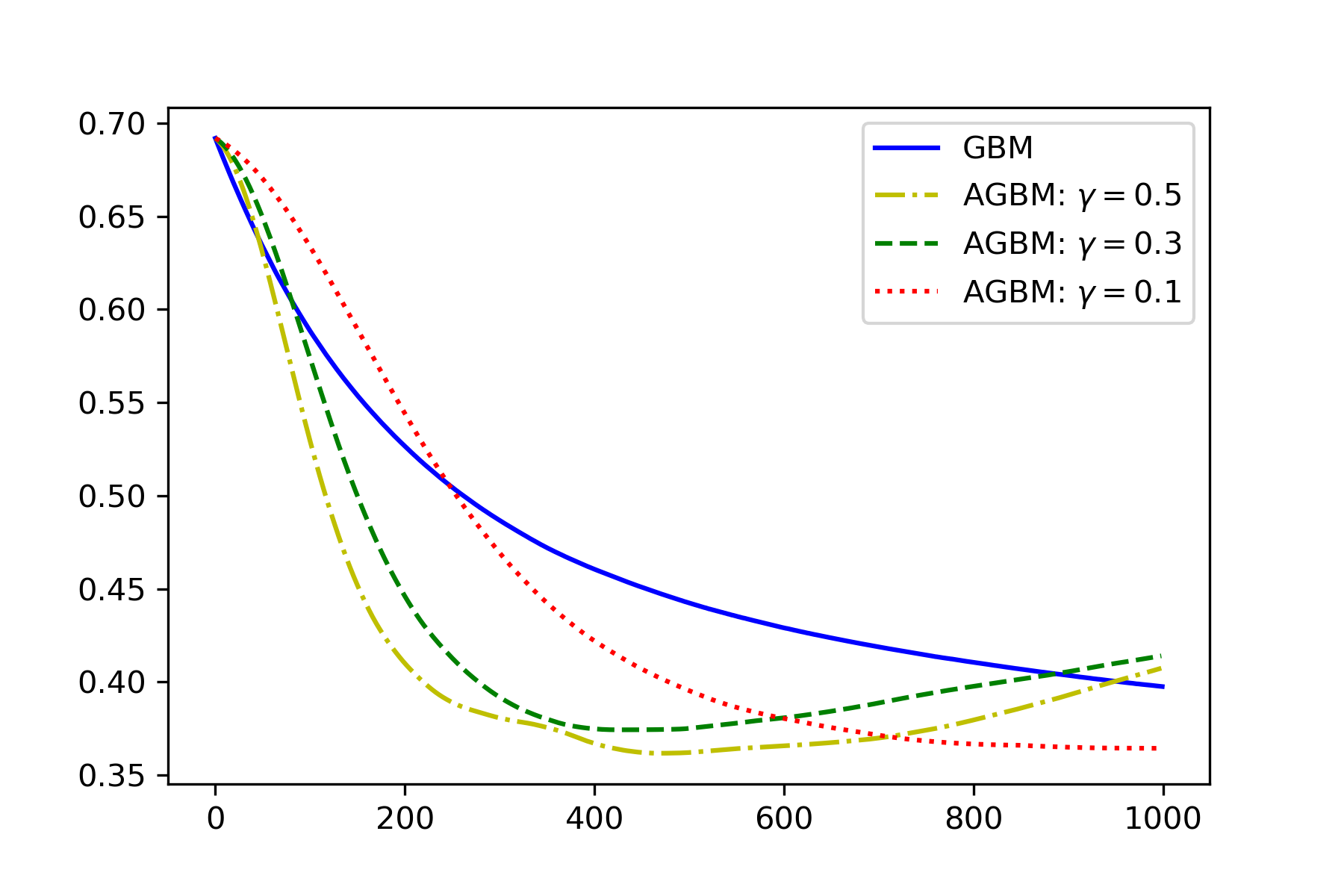}  \\
  &  {~~~number of trees} &      {~~~~~number of trees}&      {~~~~~number of trees} \\
\end{tabular}}

\caption{Training and testing loss versus number of trees for logistic regression on a1a.}
\label{fig:loss}

\end{figure*}
In this section, we present the results of computational experiments and discuss the performance of AGBM with trees as weak-learners. Subsection \ref{subsec:vagm-diverge} discusses the necessity of the error-corrected residual in AGBM. 
Subsection \ref{exp-sensitivity} shows training and testing performance for GBM and AGBM with different parameters. Subsection \ref{exp-tuned} compares the performance of GBM and AGBM with best tuned parameters. The code for the numerical experiments is available at: \url{https://github.com/google-research/accelerated_gbm}.

{\bf AGBM with CART trees}: In our experiments, all algorithms use CART trees as the weak learners. For classification problems, we use logistic loss function, and for regression problems, we use least squares loss.  To reduce the computational cost, for each split and each feature, we consider 100 quantiles (instead of potentially all $n$ values). These strategies are
commonly used in implementations of GBM like \citet{chen2016xgboost, ponomareva2017tf}. 

\subsection{Vanilla Accelerated Gradient Boosting (VAGM)}\label{subsec:vagm-diverge}
\begin{figure}[h]
\centering
{\begin{tabular}{l c c}
& $\eta=1$ & $\eta=0.3$ \\
\rotatebox{90}{ {~training loss}}&
\includegraphics[width=0.4\columnwidth, trim =1.1cm 1.0cm 1.1cm 1.0cm, clip = true]{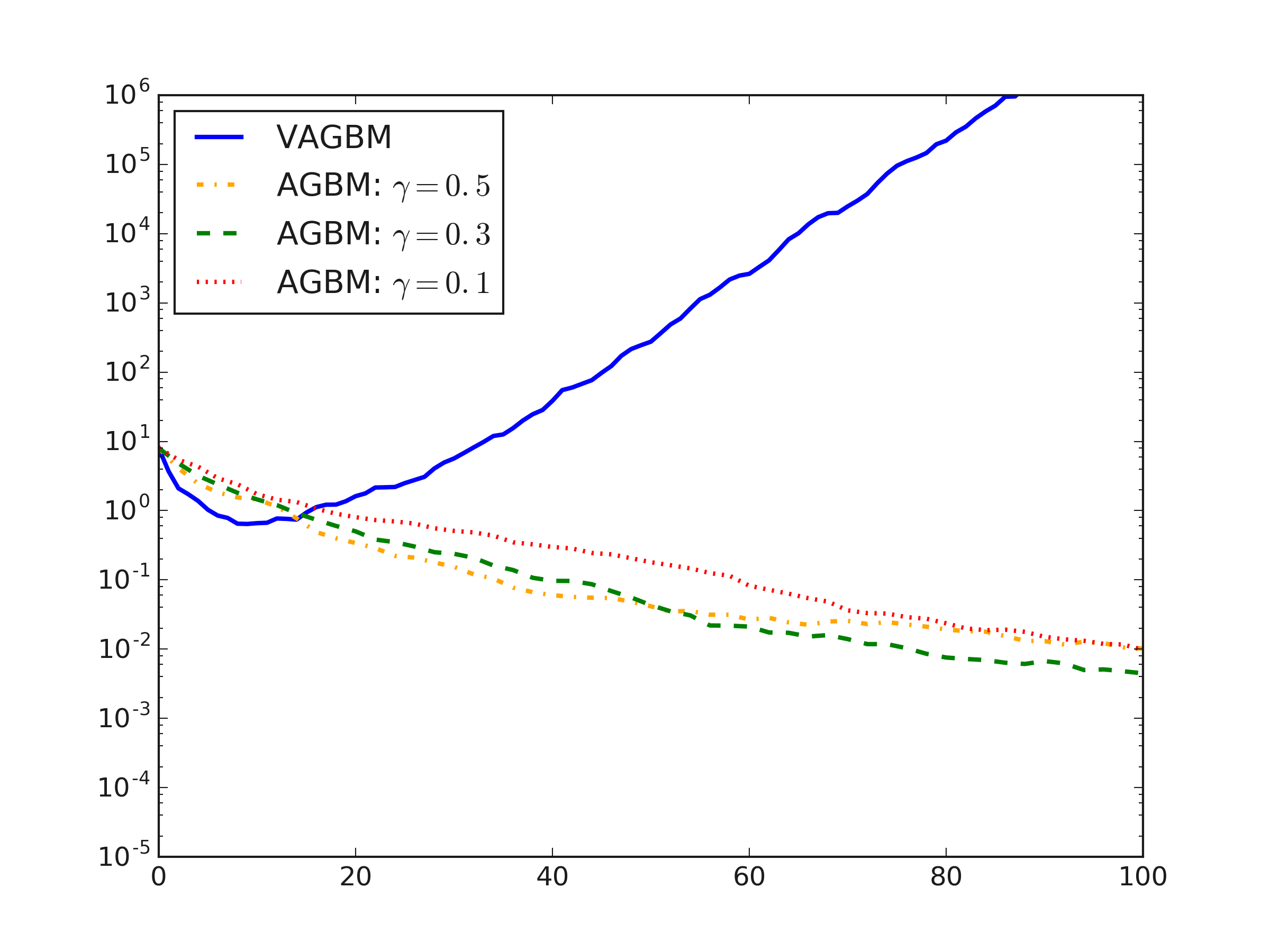} &
\includegraphics[width=0.4\columnwidth, trim =1.1cm 1.0cm 1.1cm 1.0cm, clip = true]{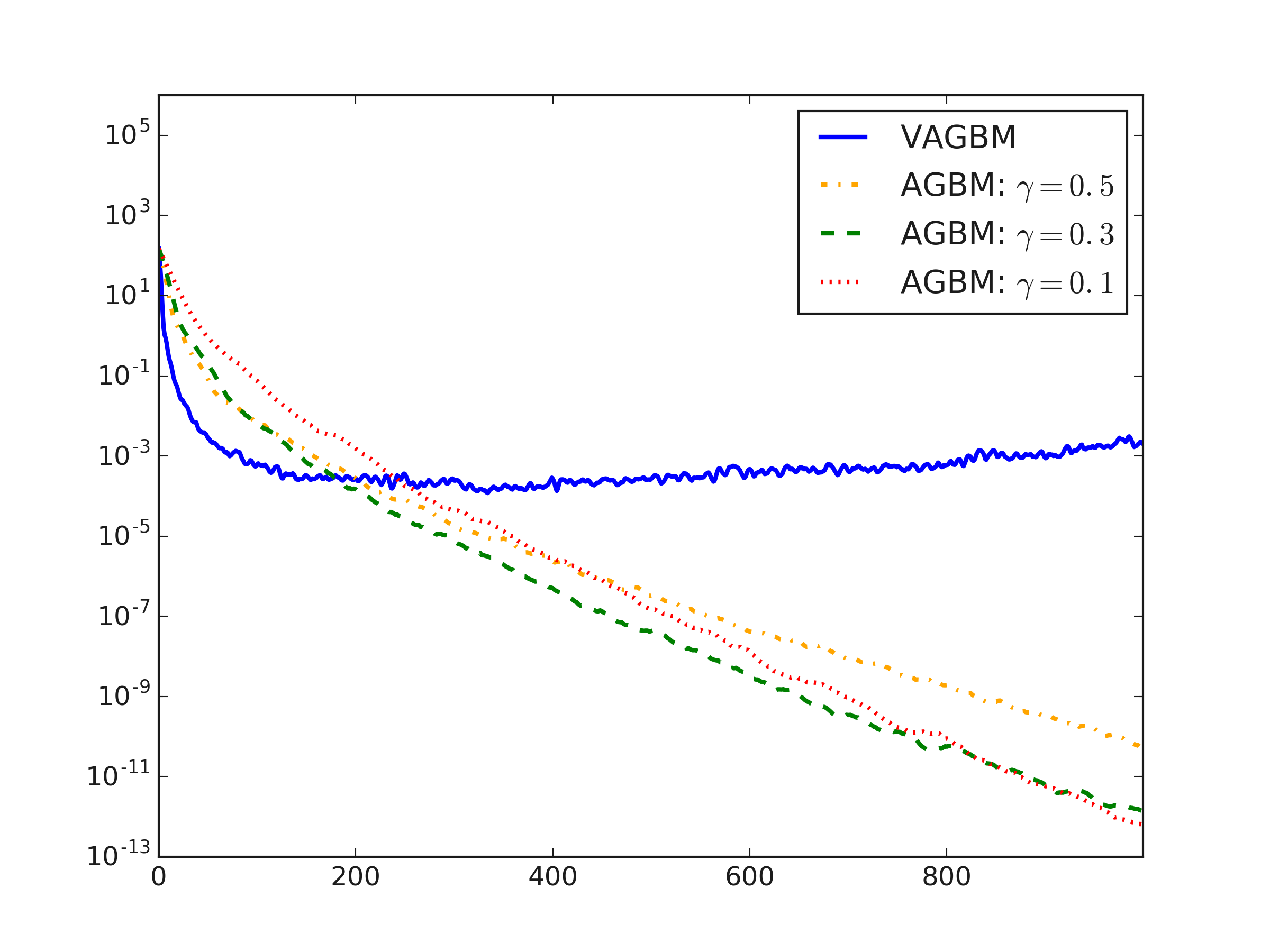}\\
  &  {~~~number of trees} & {~~~number of trees} \\
\end{tabular}}
\caption{Training loss versus number of trees for VAGBM (which doesn't converge) and AGBM with different parameters $\gamma$.}
\label{fig:VAGBM_diverge}\vspace*{-5mm}
\end{figure}
A more straightforward way of incorporating Nesterov momentum in boosting (which we refer to as vanilla AGBM or VAGBM) is explored in \citet{biau2018accelerated} and \citet{fouillen2018accelerated}. VAGBM adds only one base weak-learner in each iteration as opposed to AGBM which adds two. Unfortunately, VAGBM may not always converge to the optimum as we empirically demonstrate here. A more theoretical discussion of VAGBM is presented in Section \ref{sec:direct-acceleration}.

Figure \ref{fig:VAGBM_diverge} shows the training loss versus the number of trees for the housing dataset with step-size $\eta=1$ and $\eta=0.3$ for VAGBM and for AGBM with different parameters $\gamma$. The $x$-axis is number of trees added to the ensemble (recall that our AGBM algorithm adds two trees to the ensemble per iteration, so the number of boosting iterations of VAGBM and AGBM is different). As we can see, when $\eta$ is large, the training loss for VAGBM diverges very fast while our AGBM with proper parameter $\gamma$ converges. When $\eta$ gets smaller, the training loss for VAGBM may decay faster than our AGBM at the begining, but it gets stuck and never converges to the true optimal solution. Eventually the training loss of VAGBM may even diverge.  
On the other hand, our theory guarantees that AGBM always converges to the optimal solution.

\begin{table*}
\centering
\footnotesize
\begin{tabular}{llllll}
\multirow{2}{*}{\textbf{\# trees}} & \multirow{2}{*}{\textbf{Dataset}} & \multicolumn{2}{c}{\textbf{AGBM}} & \multicolumn{2}{c}{\textbf{GBM}} \\
                           &                          & \textit{Training}    & \textit{Testing}    & \textit{Training}    & \textit{Testing} 
                           \\ \hline
\multirow{6}{*}{30}        
& diabetes &	0.3760+/-0.0254 &	\textbf{0.5018} +/-	0.0335 & 0.5055+/-0.0084 &	0.5364+/-0.0287 \\
& german	& 0.4076+/-0.0153& \textbf{0.5308}+/-	0.0182	& 0.5319+/-0.0044&	0.5713+/-	0.0144 \\
& housing &	2.0187+/-0.2726&	7.3432+/-3.0826& 2.3173+/-0.1177 &	\textbf{4.9773}+/-2.0395 \\
& w1a &  0.1840+/-0.0013 & \textbf{0.1949}+/-	0.0093&	0.2886+/-0.0029	&0.2903+/-	0.0065        \\								
& a1a &	0.3611+/-0.0090 &	\textbf{0.4128}+/-	0.0188	&0.4647+/-0.0052 & 0.4761+/-	0.0128 \\
& sonar &	0.1864+/-0.0108 &	\textbf{0.4627}+/-	0.0548 &	0.3789+/-0.0185 &	0.5403+/-0.0367 \\ \hline
\multirow{6}{*}{50}        
& diabetes & 0.3487+/-0.0516 & \textbf{0.4869} +/-0.0390 & 0.4620 +/-	0.0060 & 0.5050+/-0.0348 
  \\
& german & 0.3695+/-0.0167&	\textbf{0.5114}+/-	0.0287&	0.4911+/-	0.0057	&0.5482+/-	0.0169\\
& housing & 1.1388+/-0.2424&	5.6229+/-	1.9212&	1.4675+/-	0.1303	&\textbf{4.7233}+/-	2.9004 \\
& w1a & 0.0743+/-0.0015 & \textbf{0.1014}+/-0.0161 &	0.2087+/-	0.0037 & 0.2121+/-0.0091\\
                           
& a1a & 0.2812+/-0.0147&	\textbf{0.3686}+/-	0.0306&	0.4144+/-	0.0063	&0.4326+/-	0.0175\\
                           
& sonar & 0.0562+/-0.0053&	\textbf{0.3768}+/-	0.0077&	0.2842+/-	0.0165	&0.4981+/-	0.0257 \\ \hline
\multirow{6}{*}{100} 
& diabetes & 0.3119+/-0.0430 & 0.4937+/-0.0459 & 0.4130+/-0.0175 & \textbf{0.4797}+/-0.0409 \\
& german  & 0.3569+/-0.0304 &	\textbf{0.5175}+/-0.0248	& 0.4364+/-	0.0089&	0.5280+/-0.0203\\
& housing &  0.6868+/-0.2020 & 5.0862+/-2.0913 &	0.8779 +/-0.1072 & \textbf{4.4168}+/-2.7163\\
& w1a & 0.0409+/-0.0034	& \textbf{0.0647}+/-0.0128&	0.1333+/-	0.0039	&0.1396+/-0.0121\\
& a1a & 0.2797+/-0.0132 & \textbf{0.3675}+/-0.0363 &	0.3575+/-	0.0057 & 0.3914+/-0.0232\\
& sonar & 0.0225+/-0.0179 & \textbf{0.3540}+/-0.0787 &	0.1902+/-	0.0637	& 0.4664+/-0.0660 \\ \hline
\end{tabular} 
\caption{Performance after tuning hyper-parameters on a representative sample of data-sets.}
\label{tab:testing}
\vspace*{-5mm} 
\end{table*}

\subsection{Typical Performance of AGBM}\label{exp-sensitivity}
Figure \ref{fig:loss} presents the training loss and the testing loss of GBM and AGBM (with three $\gamma$ values) versus the number of trees for the \textbf{a1a} dataset with three different learning rate $\eta=1$, $\eta=0.1$ and $\eta=0.01$ (recall that AGBM adds two trees per iteration). It can be seen clearly that AGBM has faster training performance than GBM for all learning rates $\eta$, empirically showcasing the difference between convergence rates of $O(1/m^2)$ and $O(1/m)$. The training loss in general decays faster with larger learning rate $\eta$. 

On test, all algorithms eventually overfit. However, AGBM can overfit in an earlier stage than GBM and seems to be more sensitive to number of trees added. This is because the training loss of AGBM decays too fast and the variance takes control in the testing loss. 
This seems to indicate that overfitting on test loss accompanies faster convergence on training loss. However, this issue can easily be circumvented by using early stopping---the best test loss of AGBM is comparable to that of GBM. In fact, AGBM with early stopping may require fewer iterations/trees than GBM to get similar training/testing performance.

\vspace{-0.1cm}
\subsection{Experiments with Fine Tuning}\label{exp-tuned}
We evaluate AGBM and GBM on a number of small datasets, fixing the number of trees, depth and $\eta$ and tuning other hyper-parameters. See Section \ref{sec:additional-experiments} for details. The results are tabulated in Table \ref{tab:testing}.
As we can see, the accelerated method in general is beneficial for underfiting scenarios (30 and 50 trees). Housing is a small dataset where AGBM seems to overfit quickly. For such small datasets, 100 weak learners start to overfit, and accelerated method overfits faster, as expected.

\vspace{-0.1cm}
\section{Additional Discussions}\vspace{-3mm}
Below we explain relevance of our results when applied to frameworks typically used in practice.
\vspace{-0.1cm}
\subsection{Use of Hessian}
Popular boosting libraries such as XGBoost \citep{chen2016xgboost} and TFBT \citep{ponomareva2017tf} compute the Hessian and perform a \emph{Newton} boosting step instead of gradient boosting. Since the Newton step may not be well defined (e.g. if the Hessian is degenerate), an additional euclidean regularizer is also added. This has been shown to improve performance and reduce the need for a line-search for the $\eta$ parameter sequence \citep{sun2014convergence, sigrist2018gradient}. For LogitBoost (i.e. when $l(x)$ is the logistic loss), \citet{sun2014convergence} demonstrate that trust-region Newton's method can indeed significantly improve the  convergence. Leveraging similar results in second-order methods for convex optimization (e.g. \citet{nesterov2006cubic,karimireddy2018global}) and adapting accelerated second-order methods \citet{nesterov2008accelerating} would be an interesting direction for the future work.
\vspace{-0.1cm}
\subsection{Out-of-sample Performance}
Throughout this work we focus only on minimizing the empirical training loss $L(f)$ (see Formula \eqref{eq:loss}). In reality what we really care about is the out-of-sample error of our resulting ensemble $f^M(x)$. A number of regularization tricks such as i) early stopping \citep{zhang2005boosting}, ii) pruning \citep{chen2016xgboost,ponomareva2017tf}, iii) smaller step-sizes \citep{ponomareva2017tf},  iv) dropout \citep{ponomareva2017tf} etc. are usually employed in practice to prevent over-fitting and improve generalization. Since AGBM requires much fewer iterations to achieve the same training loss than GBM, it outputs a much sparser set of learners. At the same time, it is common to slow down learning process (for example using smaller learning rate and weaker trees) to reduce overfitting on small dataset (but train for longer and have a larger ensemble). From preliminary experimental studies we see that AGBM overfits fast on small datasets and should be used with early stopping or more aggressive regularization. However, faster learning should be beneficial for large datasets and complex decision functions, where AGBM can deliver much smaller ensembles with good performance.
A joint theoretical study of the out-of-sample error along with the empirical error $L_n(f)$ in the style of \citet{zhang2005boosting} is much needed.
\vspace{-0.1cm}
\section{Conclusion}
In this paper, we proposed a novel Accelerated Gradient Boosting Machine (AGBM) which can be used with any type of weak learners and which provably converges faster than the traditional Gradient Boosting Machine (GBM). We also ran preliminary experiments and demonstrated that AGBM indeed converges significantly faster than GBM on the training (empirical) loss and can match or improve upon GBM test loss. In practice, however, boosting methods are equipped with a number of additional heuristics which improve the test error. A systematic analysis of such heuristics, and incorporating them into the AGBM framework are promising directions for future work.

\pagebreak

\bibliography{example_paper}
\bibliographystyle{apalike}
\onecolumn
\appendix
\part*{Appendix}

\section{Additional Experiment Details}\label{sec:additional-experiments}

{\bf Datasets}: Table \ref{tab:stats} summaries the basic statistics of the LIBSVM datasets that were used. 

\begin{table}[H]
\centering
\begin{tabular}{|c|c|c|c|}
\hline
\textbf{Dataset}              & \textbf{task}           & \textbf{\# samples}      & \textbf{\# features}   \\ \hline
a1a               & classification & 1605  & 123   \\ \hline
w1a              & classification & 2477  & 300 \\ \hline
diabetes              & classification & 768  & 8 \\ \hline
german              & classification & 1000  & 24 \\ \hline
housing & regression     & 506 & 13    \\ \hline
sonar              & classification & 208  & 60 \\ \hline
\end{tabular}
\caption{\small{Basic statistics of the (real) datasets used.}}
\label{tab:stats}
\end{table}

{\bf All fine-tuning experiments}: 
We now look at the testing performance of GBM and AGBM on six datasets with hyperparameter tuning. 

For each dataset, we randomly choose 80\% as the training and the remaining as the testing dataset. We repeat this splitting 5 times and report mean train and test errors along with standard errors. 

We consider depth $3$ trees as weak-learners and fix the number of trees to 30, 50 and 100 (notice, that for AGBM that means that the number of boosting iterations is 15, 25 and 50 respectively). We fix learning rate ($\eta$) to 0.1 and tune (using 5 fold cross-validation on training dataset with  \textit{RandomizedSearchCV} in scikit-learn) the following parameters:
\begin{itemize}
\item $min\_split\_gain$ - [10, 5, 2, 1, 0.5, 0.1, 0.01, 0.001, 1e-4, 1e-5]
\item l2 regularizer on leaves - [0.01, 0.1, 0.5, 1,2,4, 8, 16, 32, 64]
\item momentum parameter $\gamma$ (only for AGBM): uniform(0.1,1)
\end{itemize}
We use early stopping for final training on full training dataset (using 5 early stop rounds)

 As AGBM has more parameters (namely $\gamma$), we did proportionally more iterations of random search for AGBM. 
 
 As we can see from Table \ref{tab:testing}, the accelerated method in general is beneficial for underfiting scenarios (30 and 50 trees). However, for such small datasets, 100 weak learners start overfiting, and accelerated method overfits faster, as expected.

\section{Extensions and Variants}\label{sec-practical-ext}
In this section we study two more practical variants of AGBM. First we see how to restart the algorithm to take advantage of strong convexity of the loss function. Then we will study a straight-forward approach to accelerated GBM, which we call vanilla accelerated gradient boosting machine (VAGBM), a variant of the recently proposed algorithm in \citet{biau2018accelerated}, however without any theoretical guarantees.

\subsection{A Vanilla Accelerated Gradient Boosting Method}\label{sec:direct-acceleration}
A natural question to ask is whether, instead of adding \emph{two} learners at each iteration, we can get away with adding only \emph{one}? Below we show how such an algorithm would look like and argue that it may not always converge.

Following the updates in Equation \eqref{eq:(1)accel-GD-updates}, we can get a direct acceleration of GBM by using the weak learner fitting the gradient. This leads to an Algorithm \ref{al:dagbm}.

\begin{algorithm*}[h]
\caption{Vanilla Accelerated Gradient Boosting Machine (VAGBM)}\label{al:dagbm}

\begin{algorithmic}
\STATE {\bf Input.} Starting function $f^0(x)=0$, step-size $\eta$, momentum parameter $\gamma \in (0,1]$.
\STATE {\bf Initialization.}  $h^{0}(x)= f^{0}(x)$, and sequence $\theta_m=\frac{2}{m+2}$.\\
For $m=0,\ldots,M-1$ do:\\

\STATE  {\bf Perform Updates:}  

(1) Compute a linear combination of $f$ and $h$: $g^{m}(x) = (1 - \theta_m)f^{m}(x) + \theta_m h^{m}(x)$.

(2) Compute pseudo residual: $r^m=-\left[\frac{\partial \ell(y_{i},g^{m}(x_{i}))}{\partial g^{m}(x_{i})}\right]_{i=1,\ldots,n}.$

(3) Find the best weak-learner for pseudo residual: $\tau_{m} =\argmin_{\tau \in \wl_m}\sum_{i=1}^n ( r_i^m- b_\tau(x_i))^{2}$.

(4) Update the model: $f^{m+1}(x)=g^{m}(x)+\eta b_{\tau_{m}}(x)$.

(5) Update the momentum model: $h^{m+1}(x) = h^m(x) +  \eta/\theta_m b_{\tau_{m}}(x)$.
\medskip

\STATE  {\bf Output.}  $f^{M}(x)$.
\end{algorithmic}
\end{algorithm*}\medskip

Algorithm \ref{al:dagbm} is equivalent to the recently developed accelerated gradient boosting machines algorithm \citep{biau2018accelerated, fouillen2018accelerated}. Unfortunately, it \textbf{may not always converge} to an optimum or may even \textbf{diverge}. This is because $b_{\tau_m}$ from Step (2) is only an approximate-fit to $r^m$, meaning that we only take an \emph{approximate} gradient descent step. While this is not an issue in the non-accelerated version, in Step (2) of Algorithm \ref{al:dagbm}, the momentum term pushes the $h$ sequence to take a large step along the approximate gradient direction. This exacerbates the effect of the approximate direction and can lead to an additive accumulation of error as shown in \citet{devolder2014first}. In Section \ref{subsec:vagm-diverge}, we see that this is not just a theoretical concern, but that Algorithm \ref{al:dagbm} also diverges in practice in some situations.
\begin{rem}
Our \emph{corrected} residual $c^m$ in Algorithm \ref{al:agbm} was crucial to the theoretical proof of converge in Theorem \ref{thm:agbm-rate-main}. One extension could be to introduce $\gamma \in (0,1)$ in step (5) of Algorithm \ref{al:dagbm} just as in Algorithm \ref{al:agbm}. 
\end{rem}
\begin{rem} It is worth noting that Vanilla AGBM may bring good empirical performance on small datasets. We hypothesize that the accumulated error in gradient may serve as an additional regularization that slows down overfitting
\end{rem}

\section{Proof of Theorem 4.1}\label{sec:agbm-rate}
This section proves our major theoretical result in the paper:

{\bf Theorem 4.1 }
Consider Accelerated Gradient Boosting Machine (Algorithm \ref{al:agbm}). Suppose $\ell$ is $\sigma$-smooth, the step-size $\eta\le \frac{1}{\sigma}$ and the momentum parameter $\gamma\le \Theta^4/(4 + \Theta^2)$. Then for all $M\ge 0$, we have:
$$
L(f^M) - L(f^*) \le \frac{1}{2\eta \gamma (M+1)^2} \|f^*(X)\|_2^2 \ .
$$\qed


 Let's start with some new notations. Define scalar constants $s = \gamma/\Theta^2$ and
 $t := (1 - s)/2 \in (0,1)$. We mostly only need $s + t \leq 1$---the specific values of $\gamma$ and $t$ are needed only in Lemma \ref{lem:computations}.
Then define 
\[ \alpha_m := \frac{\eta \gamma}{\theta_m} = \frac{\eta s \Theta^2}{\theta_m}\,,\]
then the definitions of the sequences $\{r^m\}$, $\{c^m\}$, $\hat h^{m}(X)$ and $\{\theta_m\}$ from Algorithm 3 can be simplified as:
\begin{align*}
	\theta_m &= \frac{2}{m+2}\\
	r^m &= - \left[\frac{\partial l (y_i, g^m(x_i))}{\partial g^m(x_i)} \right]_{i= 1, \dots n}\\
	c^m &=  r^m + ({\alpha_{m-1}}/{\alpha_{m}})\encase{c^{m-1} - b_{\tau^2_{m-1}}(X)} \\
	\hat h^{m+1}(X) &=  \hat h^{m}(X) +\alpha_m r^m\,.
\end{align*}

The sequence $\hat h^{m}(X)$ is in fact closely tied to the sequence $h^{m}(X)$ as we show in the next lemma.
For notational convenience, we define $c^{-1} =  b_{\tau^2_{-1}}(X) = 0$ and similarly $\frac{\alpha_{-1}}{\theta_{-1}} = 0$ throughout the proof.
\begin{lem}\label{lem:hat-h}
	$$\hat h^{m+1}(X) = h^{m+1}(X) + \alpha_{m} (c_m - b_{\tau_{m,2}}(X)) \,.$$
\end{lem}
\begin{proof}
Observe that
\[
	\hat h^{m+1}(X) = \sum_{j=0}^{m} \alpha_j r^j \quad \text{and that}\quad 
	h^{m+1}(X) = \sum_{j=0}^m \alpha_j b_{\tau_{j,2}}(X)\,.
\]	
Then we have
\begin{align*}
	\hat h^{m+1}(X) - h^{m+1}(X) &= \sum_{j=0}^m  \alpha_j(r^j - b_{\tau_{j,2}}(X)) \\
	&= \sum_{j=0}^m  \alpha_j(r^j - \frac{\alpha_{j-1}}{\alpha_{j}}b_{\tau^2_{j-1}}(X))- \alpha_m  b_{\tau_{m,2}}(X) \\
	&= \sum_{j=0}^m  \alpha_j(c^j -\frac{\alpha_{j-1}}{\alpha_j}c^{j-1})- \alpha_m  b_{\tau_{m,2}}(X)\\
	&= \sum_{j=0}^m  (\alpha_jc^j -\alpha_{j-1}c^{j-1}) - \alpha_m  b_{\tau_{m,2}}(X)\\
	&= \alpha_m (c_m - b_{\tau_{m,2}}(X)) \,,
\end{align*}
where the third equality is due to the definition of $c^m$.
\end{proof}

Lemma \ref{lem:L-decrease} presents the fact that there is sufficient decay of the loss function:
\begin{lem}\label{lem:L-decrease}
\[
	L(f^{m+1}) \leq L(g^{m}) - \frac{\eta \Theta^2}{2}\norm{r^m}^2\,.
\]
\end{lem}
\begin{proof}
	Recall that $\tau_{m,1}$ is chosen such that
	\[
		\tau_{m,1} = \argmin_{\tau \in \wl}\norm{b_{\tau}(X) - r^m}^2 \,.
	\]
	Since the class of learners $\wl$ is \emph{scalable} (Assumption \ref{ass:scalable}), we have
	\begin{align}
		\norm{b_{\tau_{m,1}}(X) - r^m}^2  &= \min_{\tau \in \wl_m}\min_{\sigma \in \RR}\norm{ \sigma b_{\tau}(X) - r^m}^2 \nonumber\\
				&= \norm{r^m}^2 \encase{1 - \argmax_{\tau \in \wl}cos(r^m , b_{\tau}(X))^2}\nonumber\\
				&\leq \norm{r^m}^2 \encase{1 - \Theta^2} \label{eq:weak-learner-density}\,,
	\end{align}	
		where the last inequality is because of the definition of $\Theta$, and the second equality is due to the simple fact that for any two vectors $a$ and $b$,
	\[
		\min_{\sigma\in \RR}\norm{\sigma a - b}^2 = \norm{a}^2 - \max_{\sigma\in \RR}\encase{\sigma \inp{a}{b} - \frac{\sigma^2}{2}\norm{b}^2} = \norm{a}^2 - \norm{a}^2\frac{\inp{a}{b}}{\norm{a}^2\norm{b}^2}\,.
	\]
	
	Now recall that $L(f^{m+1}) = \sum_{i=1}^n l(y_i, f^{m+1}(x_i))$ and that $f^{m+1}(x) = g^{m}(x) + \eta b_{\tau_{m,1}}(x)$.
	Since the loss function $l(y_i, x)$ is $\sigma$-smooth and step-size $\eta\le \frac{1}{\sigma}$, it holds that
	\begin{align*}
		L(f^{m+1}) &= \sum_{i=1}^n l(y_i, f^{m+1}(x_i))\\
		&\leq \sum_{i=1}^n l(y_i,  g^{m}(x_i) + \eta b_{\tau_{m,1}}(x_i))\\
		&\leq \sum_{i=1}^n \encase*{ l(y_i,  g^{m}(x_i)) + \frac{\partial l (y_i, g^m(x_i))}{\partial g^m(x_i)}(\eta b_{\tau_{m,1}}(x_i))  + \frac{\sigma}{2}(\eta b_{\tau_{m,1}}(x_i))^2}\\
		&\leq \sum_{i=1}^n \encase*{ l(y_i,  g^{m}(x_i)) + \frac{\partial l (y_i, g^m(x_i))}{\partial g^m(x_i)}(\eta b_{\tau_{m,1}}(x_i))  + \frac{\eta}{2}( b_{\tau_{m,1}}(x_i))^2}\\
		&= \sum_{i=1}^n \encase*{ l(y_i,  g^{m}(x_i)) - r^m_i(\eta b_{\tau_{m,1}}(x_i))  + \frac{1}{2\eta}(b_{\tau_{m,1}}(x_i))^2}\\
		&=  L(g^{m}) - \eta \inp*{r^m}{b_{\tau_{m,1}}(X)}  + \frac{\eta}{2}\norm{b_{\tau_{m,1}}(X)}^2\\
		&=  L(g^{m}) + \frac{\eta}{2}\norm{b_{\tau_{m,1}}(X) - r^m}^2 - \frac{\eta}{2}\norm{r^m}^2\\
		&\leq L(g^{m}) - \frac{\Theta^2 \eta}{2}\norm{r^m}^2\,,
	\end{align*}
where the final inequality follows from \eqref{eq:weak-learner-density}. This furnishes the proof of the lemma.
\end{proof}

Lemma \ref{lem:convexity-L} is a basic fact of convex function, and it is commonly used in the convergence analysis in accelerated method.

\begin{lem}\label{lem:convexity-L}
	For any function $f$ and $m \geq 0$,
	\[
		L(g^m) + \theta_m\inp*{r^m}{h^m(X) - f(X)} \leq \theta_m L(f) + (1 - \theta_m)L(f^m) \,.
	\]
\end{lem}
\begin{proof}
	For any function $f$, it follows from the convexity of the loss function $l$ that
	\begin{align}
	L(g^m) + \inp*{r^m}{g^m(X) - f(X)} &= \sum_{i=1}^n l(y_i,g^m(x_i)) + \frac{\partial l (y_i, g^m(x_i))}{\partial g^m(x_i)}(f(x_i) - g^m(x_i))\nonumber\\
	&\leq \sum_{i=1}^n l(y_i,f(x_i)) = L(f)\,. \label{eq:convexity-L}
	\end{align}
	Substituting $f = f^m$ in \eqref{eq:convexity-L}, we get
	\begin{equation}
		L(g^m) + \inp*{r^m}{g^m(X) - f^m(X)} \leq L(f^m)\,. 	\label{eq:convexity-L-2}
	\end{equation}
	Also recall that $g^m(X) = (1 - \theta_m)f^m(X) + \theta_m h^m(X)$. This can be rewritten as 
	\begin{equation}
		\theta_m (g^m(X) - h^m(X)) = (1 - \theta_m)(f^m(X) - g^m(X))\,.\label{eq:convexity-L-3}			
	\end{equation}
	Putting \eqref{eq:convexity-L}, \eqref{eq:convexity-L-2}, and \eqref{eq:convexity-L-3} together:
	\begin{align*}
		&L(g^m) + \theta_m\inp*{r^m}{h^m(X) - f(X)} \\
		=& L(g^m) + \theta_m\inp*{r^m}{g^m(X) - f(X)} + \theta_m\inp*{r^m}{h^m(X) - g^m(X)}\\
		=& \theta_m [L(g^m) + \inp*{r^m}{g^m(X) - f(X)}] + (1 - \theta_m)[L(g^m) +\inp*{r^m}{g^m(X) - f^m(X)}]\\
		\leq& \theta_m L(f) + (1 -\theta_m) L(f^m)\,,
	\end{align*}
	which finishes the proof.
\end{proof}
We are ready to prove the key lemma which gives us the accelerated rate of convergence.
\begin{lem}\label{lem:potential-decrease}
Define the following potential function $V(f)$ for any given output function $f$:
\begin{equation}
	V^{m}(f) = \frac{\alpha_{m-1}}{\theta_{m-1}} \encaser{L(f^{m}) - L (f)} + \frac{1}{2}\norm*{f(X) - \hat h^m(X)}^2\,. \label{eq:def-Vm}
\end{equation}
At every step, the potential decreases at least by $\delta_m$:
\[
	V^{m+1}(f) \leq V^{m}(f) + \delta_m\,,
\]
where $\delta_m$  is defined as:
\begin{equation}
	\delta_m := \frac{s \alpha_{m-1}^2}{2t}\norm{c^{m-1} - b_{\tau^2_{m-1}}(X) }^2 - (1 - s -t) \frac{\alpha_m^2}{2s}\norm{r^m}^2 \,. \label{eq:def-delta}
\end{equation}
\end{lem}
\begin{proof}
Recall that $c^{-1} =  b_{\tau^2_{-1}}(X)) = 0$ and $\frac{\alpha_{-1}}{\theta_{-1}} = 0$.
	It follows from Lemma \ref{lem:L-decrease} that:
	\begin{align*}
		&L(f^{m+1}) - L(g^{m}) + \frac{(1-s)\eta \Theta^2}{2}\norm{r^m}^2\\
		\leq& -\frac{s\eta \Theta^2}{2}\norm{r^m}^2\\
		=& -{\alpha_m \theta_m}\norm{r^m}^2 + \frac{\alpha_m \theta_m}{2}\norm{r^m}^2\\
		=& \theta_m\inp*{r^m}{\hat h^{m}(X) - \hat h^{m+1}(X)} + \frac{\theta_m}{2\alpha_m}\norm{\hat h^{m}(X) - \hat h^{m+1}(X)}^2\\
		=& \theta_m\inp*{r^m}{\hat h^{m}(X) - f(X)} + \frac{\theta_m}{2\alpha_m}\encaser{\norm{f(X) - \hat h^{m}(X)}^2 - \norm{f(X) - \hat h^{m+1}(X)}^2}\,,
	\end{align*}
	where the second equality is by the definition of $\hat h^m(x)$ and the third is just mathematical manipulation of the equation (it is also called three-point property).
	By rearranging the above inequality, we have
	\begin{align*}
		& L(f^{m+1})  + \frac{(1-s)\eta \Theta^2}{2}\norm{r^m}^2 \\ 
		\leq &  L(g^{m}) + \inp*{r^m}{\hat h^{m}(X) - f(X)}+ \frac{\theta_m}{2\alpha_m}\encaser{\norm{f(X) - \hat h^{m}(X)}^2 - \norm{f(X) - \hat h^{m+1}(X)}^2}\\
		= & L(g^{m}) +\theta_m\inp*{r^m}{ h^{m}(X) - f(X)} + \frac{\theta_m}{2\alpha_m}\encaser{\norm{f(X) - \hat h^{m}(X)}^2 - \norm{f(X) - \hat h^{m+1}(X)}^2} \\
		& \ \ \  + \theta_m\inp*{r^m}{ \hat h^{m}(X) - h^m(X)}\\
		\leq & \theta_m L(f) + (1 - \theta_m)L(f^m) + \frac{\theta_m}{2\alpha_m}\encaser{\norm{f(X) - \hat h^{m}(X)}^2 - \norm{f(X) - \hat h^{m+1}(X)}^2} \\
		& \ \ \ + \theta_m \alpha_{m-1}\inp*{r^m}{c^{m-1} - b_{\tau^2_{m-1}}(X)}\,,
	\end{align*}
	where the first inequality uses Lemma \ref{lem:convexity-L} and the last inequality is due to the fact that $ \hat h^{m}(X) - h^m(X) = \alpha_{m-1} (c^{m-1} - b_{\tau^2_{m-1}}(X))$ from Lemma \ref{lem:hat-h}. Rearranging the terms and multiplying by $(\alpha_m/\theta_m)$ leads to
	\begin{multline*}
		\frac{\alpha_m}{\theta_m}(L(f^{m+1}) - L(f)) + \frac{1}{2}\norm{f(X) - \hat h^{m+1}(X)}^2  \\   \leq \underbrace{\frac{\alpha_m (1 - \theta_m)}{\theta_m}}_{:=\mathcal{A}}(L(f^{m}) - L(f)) + \frac{1}{2}\norm{f(X) - \hat h^{m}(X)}^2 + \underbrace{\alpha_m\alpha_{m-1}\inp*{r^m}{(c^{m-1} - b_{\tau^2_{m-1}}(X))} - \frac{(1 - s)\eta \Theta^2 \alpha_m}{2\theta_m}\norm{r^m}^2 }_{:=\mathcal{B}}\,.
	\end{multline*}
	Let us examine first the term $\mathcal{A}$:
	\[
		\frac{\alpha_m (1 - \theta_m)}{\theta_m} =  (\eta \Theta^2 s)\frac{1 - \theta_m}{\theta_m^2} \leq (\eta \Theta^2 s)\frac{1}{\theta_{m-1}^2} = \frac{\alpha_{m-1}}{\theta_{m-1}}\,.
	\]
	We have thus far shown that 
	\[
		V^{m+1}(f) \leq V^{m}(f) + \mathcal{B}\,,
	\]
	and we now need to show that $\mathcal{B} \leq \delta_m$.
	Using Mean-Value inequality, the first term in $\mathcal{B}$ can be bounded as
	\begin{align*}
	\alpha_m \alpha_{m-1}\inp*{r^m}{(c^{m-1} - b_{\tau^2_{m-1}}(X))} &\leq \frac{\alpha_m^2 t}{2s}\norm{r^m}^2 + \frac{\alpha_{m-1}^2 s}{2t}\norm{c^{m-1} - b_{\tau^2_{m-1}}(X)}^2 \,.
	\end{align*}
	Substituting it in $\mathcal{B}$ shows:
	\begin{align*}
		\mathcal{B} &= \alpha_m\alpha_{m-1}\inp*{r^m}{(c^{m-1} - b_{\tau^2_{m-1}}(X))} - \frac{(1 - s)\eta \Theta^2 \alpha_m}{2\theta_m}\norm{r^m}^2\\
		&\leq \frac{\alpha_m^2 t}{2s}\norm{r^m}^2 + \frac{\alpha_{m-1}^2 s}{2t}\norm{c^{m-1} - b_{\tau^2_{m-1}}(X)}^2 - \frac{(1 - s)\alpha_m^2}{2s}\norm{r^m}^2\\
		&= \frac{\alpha_{m-1}^2 s}{2t}\norm{c^{m-1} - b_{\tau^2_{m-1}}(X)}^2 -  (1 - s - t)\frac{\alpha_m^2}{2s}\norm{r^m}^2 \\
		&= \delta_m\,,
	\end{align*}
	which finishes the proof.
\end{proof}
Unlike the typical proofs of accelerated algorithms, which usually shows that the potential $V^m(f)$ is a decreasing sequence, there is no guarantee that the potential $V^m(f)$ is decreasing in the boosting setting due to the use of weak learners. Instead, we are able to prove that:

\begin{lem}\label{lem:non-negative}
For any given $m$, it holds that $\sum_{j=0}^m\delta_j \leq 0$.
\end{lem}

\begin{proof}
We can rewrite the statement of the lemma as:
\begin{equation}
	\sum_{j=0}^{m-1} \alpha_j^2\norm{c^{j} - b_{\tau_{j,2}}(X)}^2 \leq \frac{t(1 - s - t)}{s^2} \sum_{j=0}^{m}\alpha_{j}^2\norm{r^j}^2\,.\label{eq:required-inequality-deltas}
\end{equation}

Here, let us focus on the term $\norm{c^{j+1} - b_{\tau^2_{j+1}}(X)}^2$ for a given $j$. We have that
\begin{align*}
	\norm*{c^{j+1} - b_{\tau^2_{j+1}}(X)}^2 &\leq (1 - \Theta^2)\norm*{c^{j+1}}^2\\
	&= (1 - \Theta^2)\norm*{r^{j+1} + \frac{\theta_{j+1}}{\theta_{j}}\encase{c^{j} - b_{\tau_{j, 2}}(X)}}^2\\
	&\leq (1 - \Theta^2)(1 + \rho)\norm*{r^{j+1}}^2 + (1 - \Theta^2)(1 + 1/\rho)\norm*{\frac{\theta_{j+1}}{\theta_{j}}\encase{c^{j} - b_{\tau_{j, 2}}(X)}}^2\\
	&\leq (1 + \rho)(1 - \Theta^2)\norm*{r^{j+1}}^2 + (1 - \Theta^2)(1 + 1/\rho)\norm*{\encase{c^{j} - b_{\tau_{j, 2}}(X)}}^2\,,
\end{align*}
where the first inequality follows from our assumption about the density of the weak-learner class $\BB$ (the same of the argument in \eqref{eq:weak-learner-density}),
the second inequality holds for any $\rho \geq 0$ due to Mean-Value inequality, and the last inequality is from $\theta_{j+1} \leq \theta_j$. We now derives a recursive bound on the left side of \eqref{eq:required-inequality-deltas}. From this, \eqref{eq:required-inequality-deltas} follows from an elementary fact of recursive sequence as stated in Lemma \ref{lem:computations} with $a_j = \alpha_j^2\norm*{c^{j} - b_{\tau_{j, 2}}(X)}^2$ and $c_j = \alpha_j^2\norm*{r^j}^2$.
\end{proof}

\begin{rem}
If $c^m = b_{\tau_{m, 2}}(X)$ (i.e. our class of learners $\BB$ is \emph{strong}), then $\delta_m = -(1 - s -t)\frac{\alpha_m^2}{2s^2}\norm{r^m}^2 \leq 0$.
\end{rem}

Lemma \ref{lem:computations} is an elementary fact of recursive sequence used in the proof of Lemma \ref{lem:non-negative}.
\begin{lem}\label{lem:computations}
	Given two sequences $\{a_j \geq 0\}$ and $\{c_j \geq 0\}$ such that the following holds for any $\rho \geq 0$,
	\[
		a_{j+1} \leq (1 - \Theta^2) [(1 + 1/\rho)a_j + (1 + \rho)c_{j+1}]\,,
	\]
	then the sum of the terms $a_j$ can be bounded as
	\[
		\sum_{j=0}^m a_j \leq \frac{t(1-s-t)}{s^2}\sum_{j=0}^m c_j\,.
	\]
\end{lem}

\begin{proof}
	The recursive bound on $a_j$ implies that
	\begin{align*}
		a_{j} &\leq (1 - \Theta^2) [(1 + 1/\rho)a_{j-1} + (1 + \rho)c_{j}]\\
		&\leq \sum_{k=0}^{j} [(1 + 1/\rho)(1 - \Theta^2)]^{j-k}(1 + \rho)(1 - \Theta^2)c_{k}\,.
	\end{align*}
	Summing both the terms gives
	\begin{align*}
		\sum_{j=0}^m a_{j} &\leq \sum_{j=0}^m\sum_{k=0}^{j} [(1 + 1/\rho)(1 - \Theta^2)]^{j-k}(1 + \rho)(1 - \Theta^2)c_{k}\\
		&= \sum_{k=0}^{m} \sum_{j=k}^m[(1 + 1/\rho)(1 - \Theta^2)]^{j-k}(1 + \rho)(1 - \Theta^2)c_{k}\\
		&\leq \sum_{k=0}^{m} \left({\sum_{j=0}^\infty[(1 + 1/\rho)(1 - \Theta^2)]^{j}}\right)(1 + \rho)(1 - \Theta^2)c_{k}\\
		&= \frac{(1 + \rho)(1 - \Theta^2)}{1 - (1 + 1/\rho)(1 - \Theta^2)} \sum_{k=0}^{m} c_k\\
		&= \frac{(1 + \rho)(1 - \Theta^2)}{\Theta^2 - (1 - \Theta^2)/\rho} \sum_{k=0}^{m} c_k\\
		&= \frac{2(1 + \rho)(1 - \Theta^2)}{\Theta^2} \sum_{k=0}^{m} c_k\\
		&= \frac{2(2 - \Theta^2)(1 - \Theta^2)}{\Theta^4} \sum_{k=0}^{m} c_k\,,
	\end{align*}	
	where in the last two equalities we chose $\rho = \frac{2(1- \Theta^2)}{\Theta^2}$. Now recall that $s \leq \frac{\Theta^2}{4 + \Theta^2} \in (0,1)$ and that $t = (1 - s)/2$:
	\begin{align*}
	\sum_{j=0}^m a_{j} &\leq \frac{2(2 - \Theta^2)(1 - \Theta^2)}{\Theta^4} \sum_{k=0}^{m} c_k\\
	&\leq \frac{4}{\Theta^4} \sum_{k=0}^{m} c_k\\
	&= \left(\frac{4 + \Theta^2}{\Theta^2} - 1 \right)^2\frac{1}{4}\sum_{k=0}^{m} c_k\\
	&\leq \left(\frac{1}{s} - 1 \right)^2\frac{1}{4}\sum_{k=0}^{m} c_k\\
	&= \frac{(1 - s)^2}{4 s^2}\sum_{k=0}^{m} c_k\\
	&= \frac{t(1 - s - t)}{s^2}\sum_{k=0}^{m} c_k\,.
	\end{align*}
\end{proof}

Lemma \ref{lem:potential-decrease} and Lemma \ref{lem:non-negative} directly result in our major theorem:

\textit{Proof of Theorem 4.1}
It follows from Lemma \ref{lem:potential-decrease} and Lemma \ref{lem:non-negative} that
\[
	V^{M}(f^\star) \leq V^{M-1}(f^\star) + \delta_m \leq V^0(f^\star) + \sum_{j=0}^{M-1} \delta_j \leq \frac{1}{2}\norm{f^{0}(X) - f^\star(X)}^2\,.
\]
Notice $V^{M}(f^\star) \geq \frac{\alpha_{m-1}}{\theta_{m-1}}(L(f^{M}) - L(f^\star))$ as the term $\frac{1}{2}\norm{f^{M}(X) - f^\star(X)}^2 \geq 0$, which induces that
\[
	L(f^{M}) - L(f^\star) \leq \frac{\theta_{M-1}}{2\alpha_{M-1}}\norm{f^{0}(X) - f^\star(X)}^2 = \frac{1}{2\gamma \eta}\cdot \frac{\norm{f^{0}(X) - f^\star(X)}^2}{M^2}\,.
\]

\qed

\section{Additional Numerical Experiments}
\subsection{VAGBM may diverge with small $\eta$}
Figure \ref{fig:diverge} shows that for smaller $\eta$, VAGBM may still diverge. Of course, the smaller the $\eta$, the longer VAGBM stay stable.
\begin{figure*}
\centering
{\begin{tabular}{c c}
Training Loss & Testing Loss\\

\includegraphics[width=0.4\textwidth , clip = true]{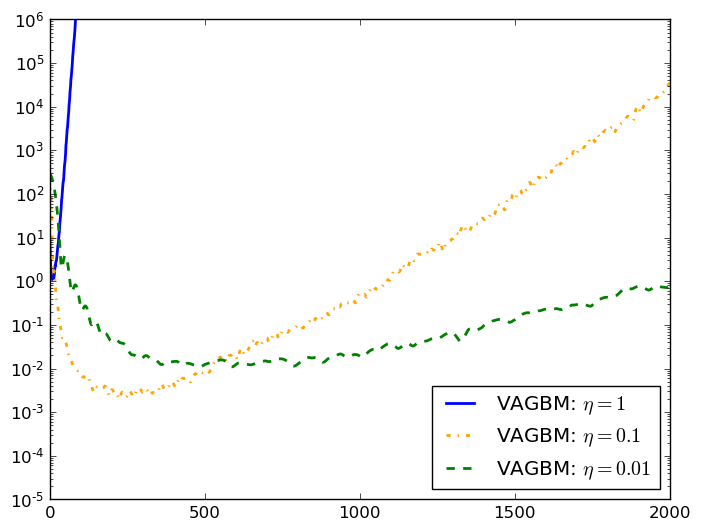}&
\includegraphics[width=0.4\textwidth, clip = true]{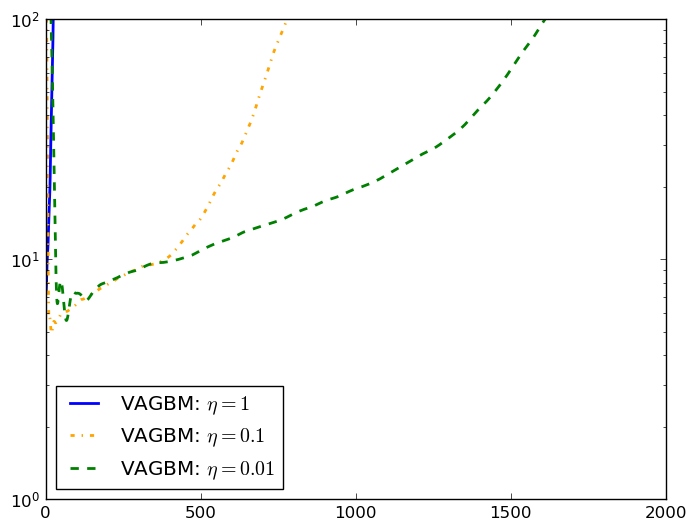}
\end{tabular}}

\caption{Training and testing loss versus number of trees for AGBM with different $\eta$.}

\label{fig:diverge}
\end{figure*}
\subsection{Performance of different algorithms with tree stumps}
Figure \ref{fig:moretest} presents the performance of different algorithms on tree stumps (namely smaller $\Theta$). They are consistent with Figure \ref{fig:loss}.
\begin{figure*}
\centering
{\begin{tabular}{l c c c}
&{   {$\eta=1$}}&{   {$\eta=0.1$}}&{   {$\eta=0.01$}}\\

\rotatebox{90}{ {~~~~~~training loss}}&
\includegraphics[width=0.3\textwidth,   trim =1cm 0.6cm 1.0cm 1.0cm, clip = true]{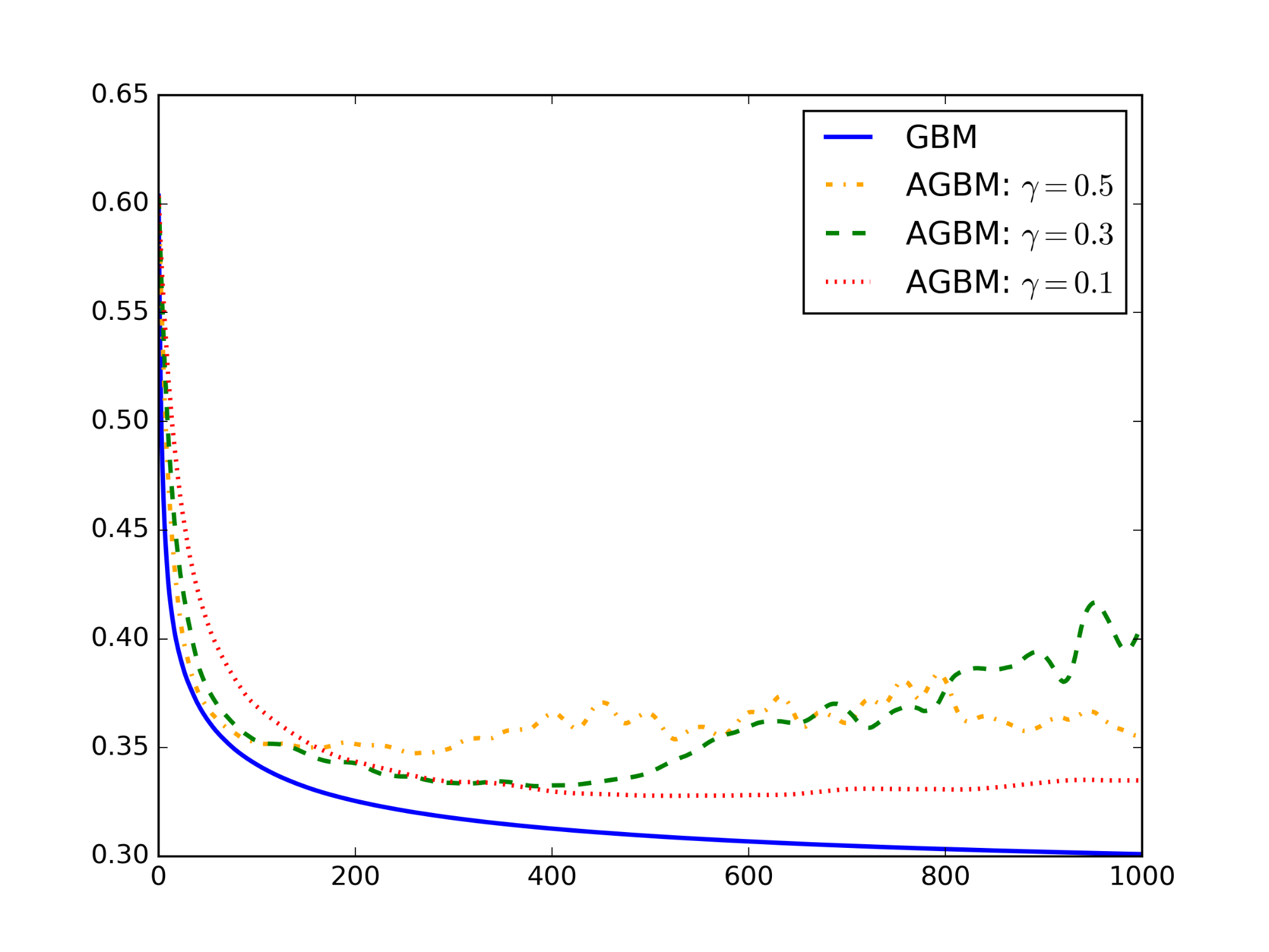}&
\includegraphics[width=0.3\textwidth,  trim =1cm 0.6cm 1.0cm 1.0cm, clip = true]{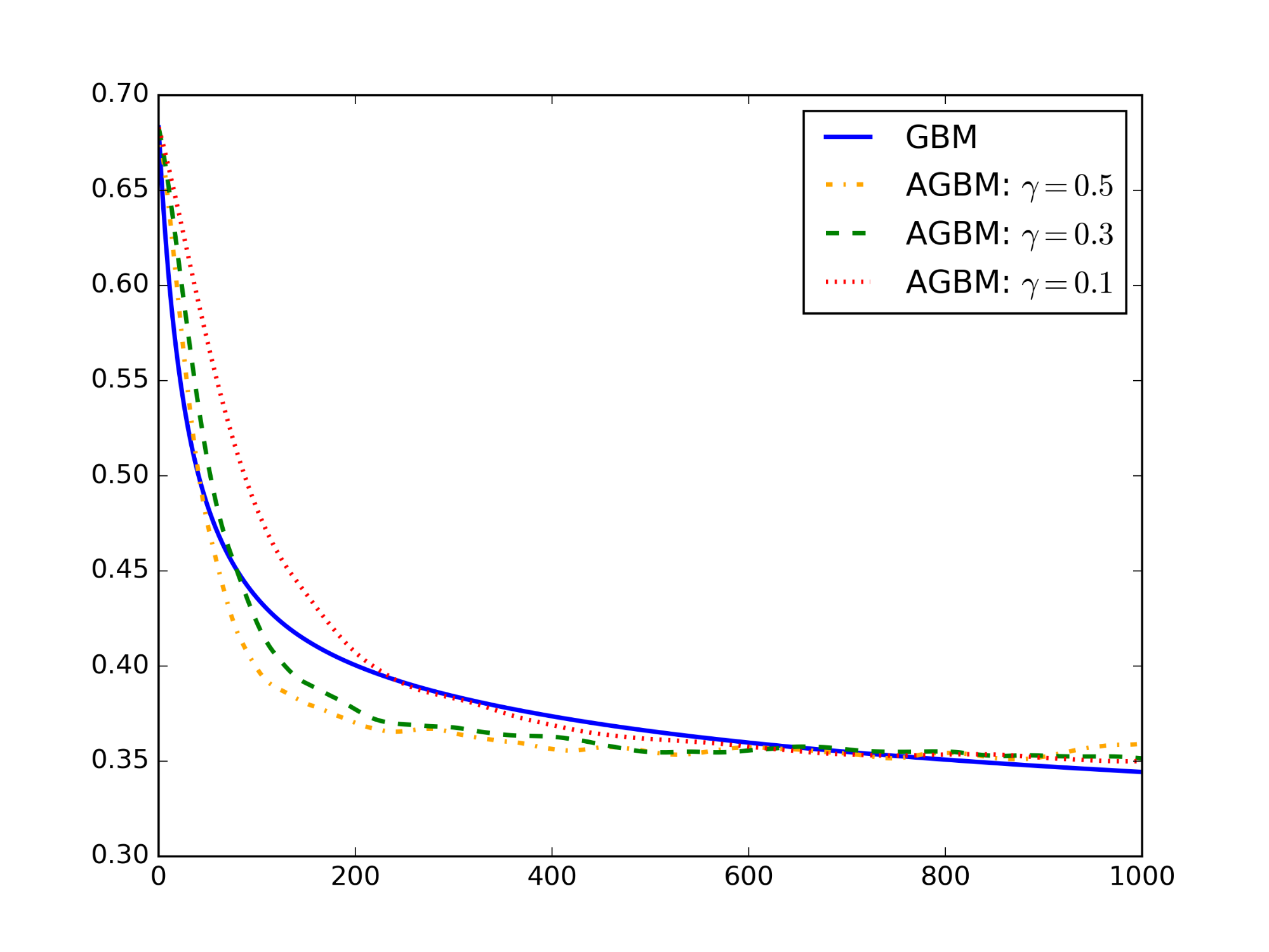} &
\includegraphics[width=0.3\textwidth,  trim =1cm 0.6cm 1.0cm 1.0cm, clip = true]{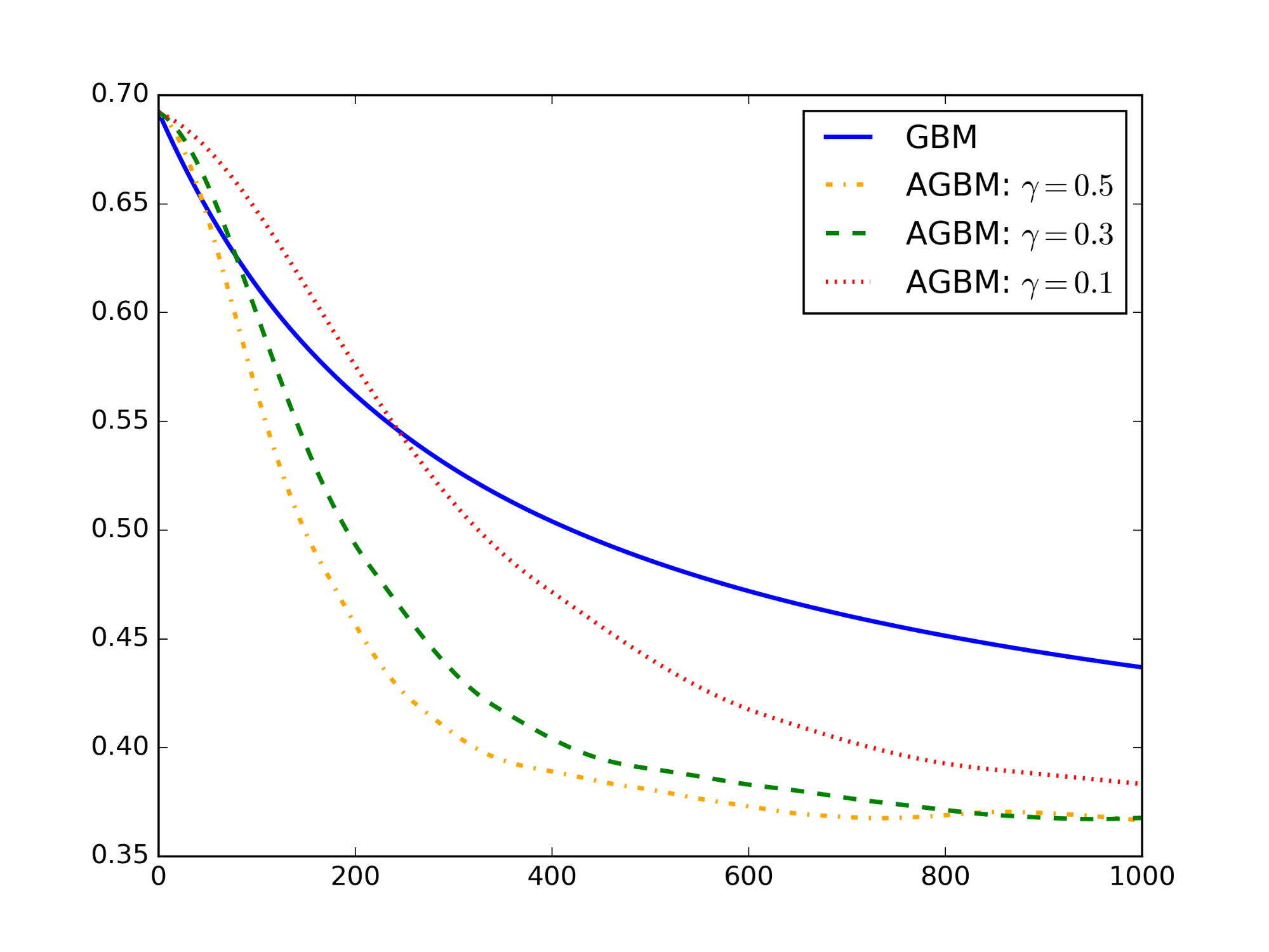} \\

\rotatebox{90}{ {~~~~~~testing loss}}&
\includegraphics[width=0.3\textwidth, trim =1cm 0.6cm 1.0cm 1.0cm, clip = true]{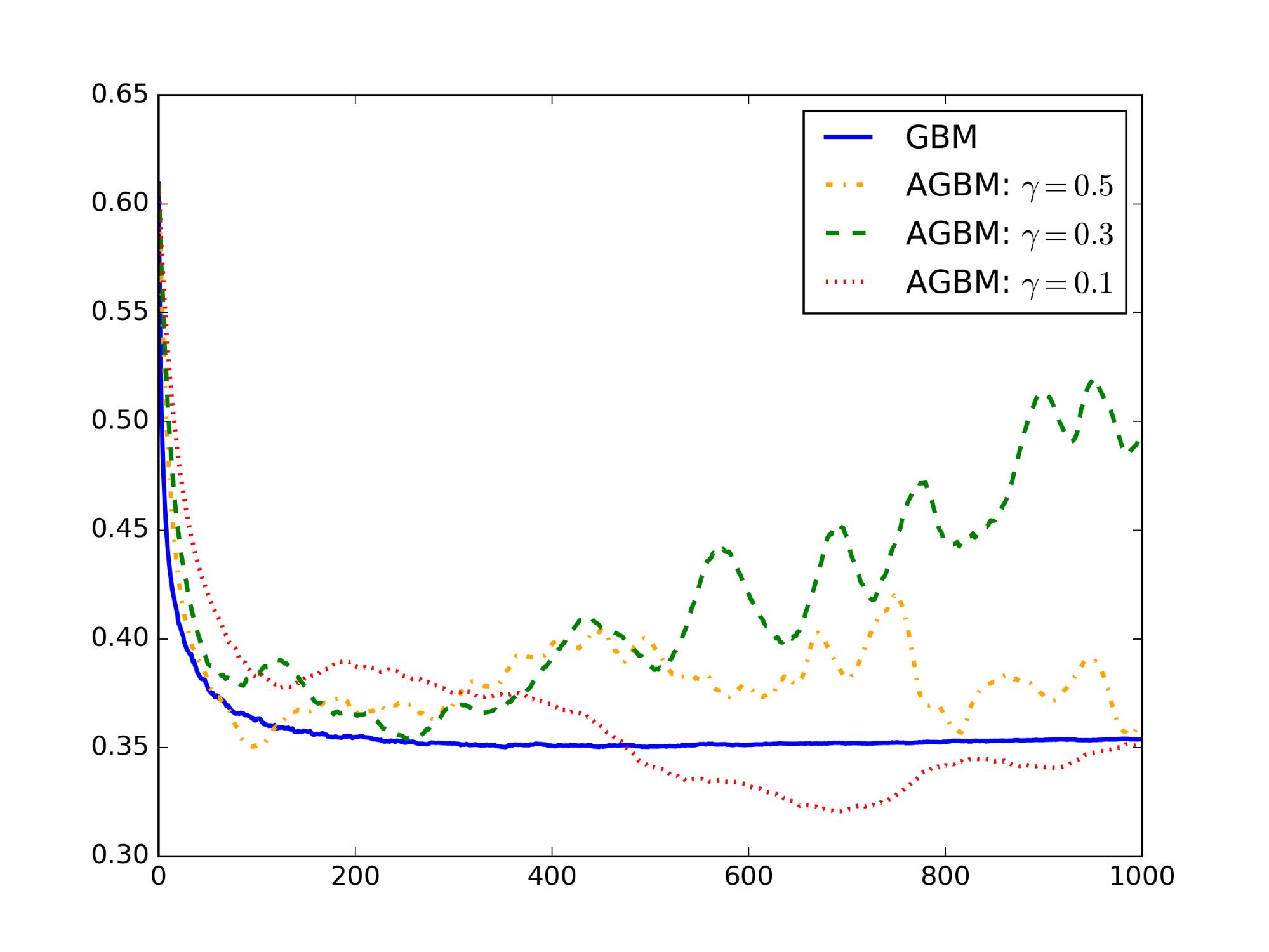}&

\includegraphics[width=0.3\textwidth, trim =1cm 0.6cm 1.0cm 1.0cm, clip = true]{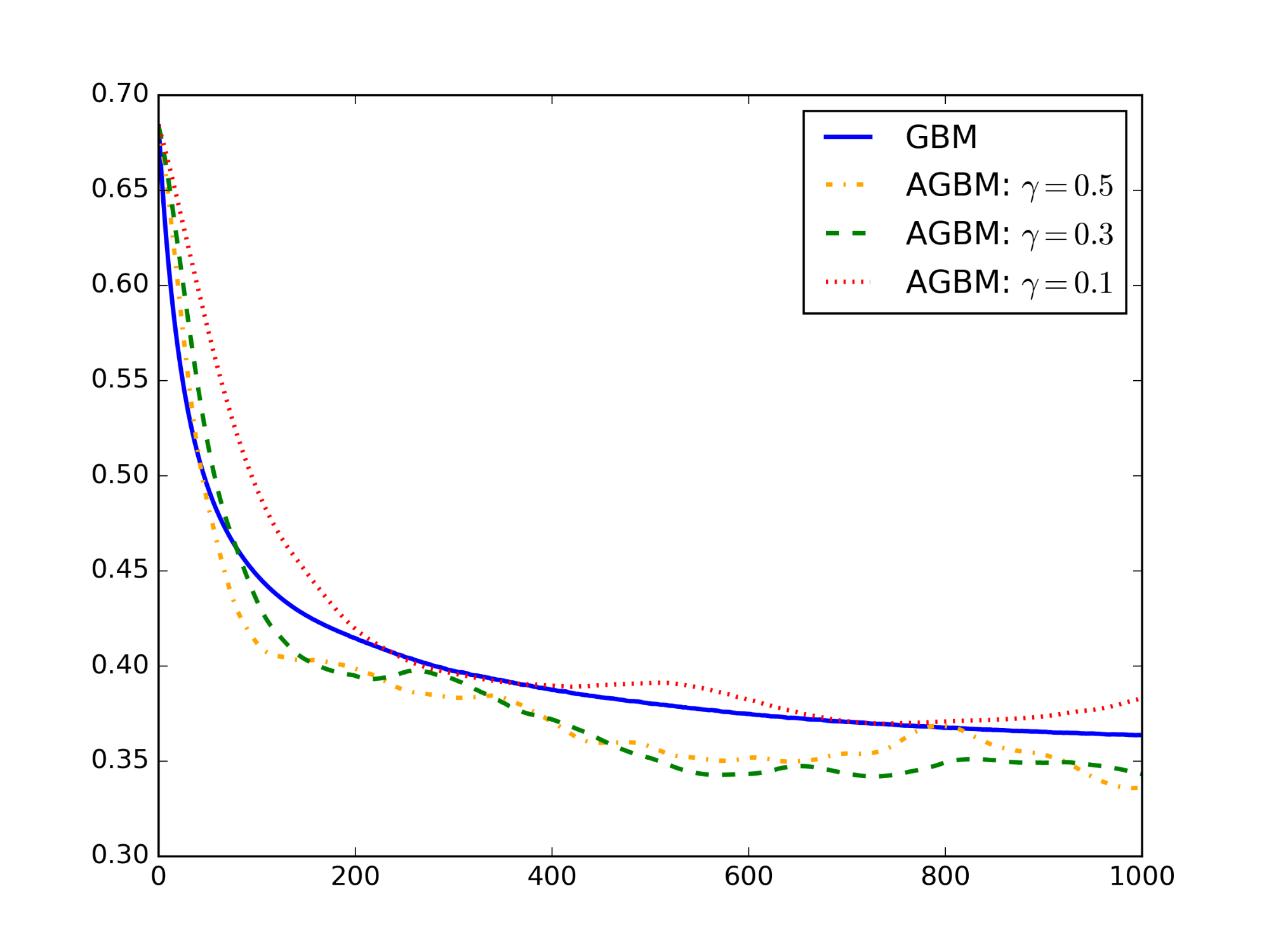}&

\includegraphics[width=0.3\textwidth, trim =1cm 0.6cm 1.0cm 1.0cm, clip = true]{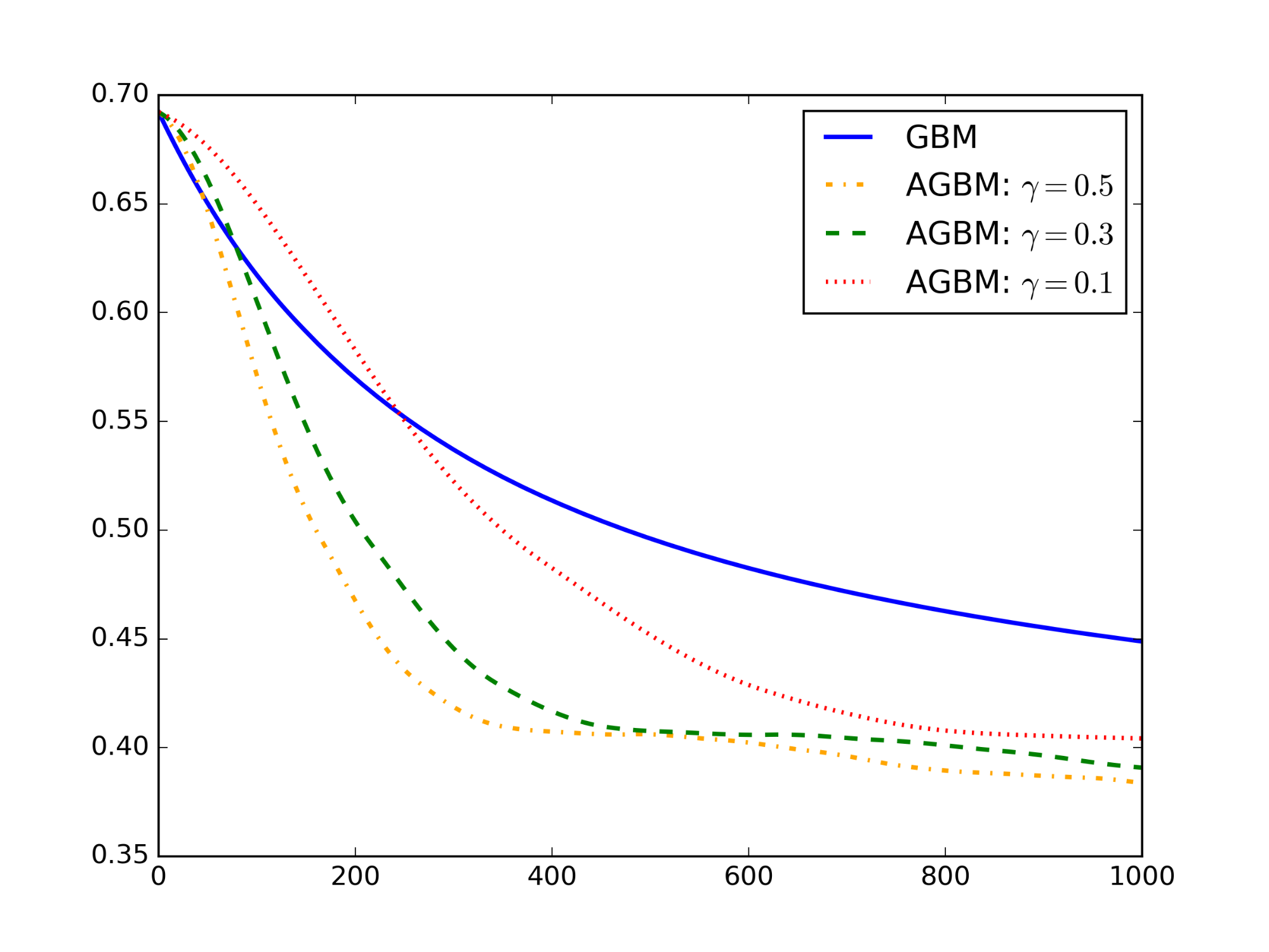}  \\
  &  {~~~number of trees} &      {~~~~~number of trees}&      {~~~~~number of trees} \\
\end{tabular}}

\caption{Training and testing loss versus number of trees for logistic regression on a1a with tree stumps (one layer decision trees).}
\label{fig:moretest}
\end{figure*}
\end{document}